\documentclass{article}
\pdfoutput=1
\usepackage[latin9]{inputenc}
\usepackage{geometry}
\usepackage{color}
\usepackage{amsthm}
\usepackage{amsmath}
\usepackage{graphicx}
\usepackage{esint}

\makeatletter
  \theoremstyle{remark}
  \newtheorem*{claim*}{\protect\claimname}

\usepackage{url}
\usepackage{amssymb}

\@ifundefined{showcaptionsetup}{}{%
 \PassOptionsToPackage{caption=false}{subfig}}
\usepackage{subfig}
\makeatother

  \providecommand{\claimname}{Claim}

\begin{document}

\title{An Analytically Tractable Bayesian Approximation to Optimal Point
Process Filtering}

\author{Yuval Harel\\
Department of Electrical Engineering\\
Technion -- Israel Institute of Technology\\
Technion City, Haifa, Israel\\
\href{mailto:yharel@tx.technion.ac.il}{yharel@tx.technion.ac.il}\\
\and Ron Meir\\
Department of Electrical Engineering\\
Technion -- Israel Institute of Technology\\
Technion City, Haifa, Israel\\
\href{mailto:rmeir@ee.technion.ac.il}{rmeir@ee.technion.ac.il}\\
\and Manfred Opper\\
Department of Artifcial Intelligence\\
Technical University Berlin\\
Berlin 10587, Germany\\
\href{mailto:rmeir@ee.technion.ac.il}{opperm@cs.tu-berlin.de}}
\maketitle
\begin{abstract}
The process of dynamic state estimation (filtering) based on point
process observations is in general intractable. Numerical sampling
techniques are often practically useful, but lead to limited conceptual
insight about optimal encoding/decoding strategies, which are of significant
relevance to Computational Neuroscience. We develop an analytically
tractable Bayesian approximation to optimal filtering based on point
process observations, which allows us to introduce distributional
assumptions about sensory cell properties, that greatly facilitates
the analysis of optimal encoding in situations deviating from common
assumptions of uniform coding. The analytic framework leads to insights
which are difficult to obtain from numerical algorithms, and is consistent
with experiments about the distribution of tuning curve centers. Interestingly,
we find that the information gained from the absence of spikes may
be crucial to performance. 
\end{abstract}

\section{Introduction}

The task of inferring a hidden dynamic state based on partial noisy
observations plays an important role within both applied and natural
domains. A widely studied problem is that of online inference of the
hidden state at a given time based on observations up to to that time,
referred to as \emph{filtering} \cite{AndMoo05}. For the linear setting
with Gaussian noise and quadratic cost, the solution is well known
since the early 1960s both for discrete and continuous times, leading
to the celebrated Kalman and the Kalman-Bucy filters \cite{KalBuc61,Kalman60},
respectively. In these cases the exact posterior distribution is Gaussian,
resulting in closed form recursive update equations for the mean and
variance of this distribution, implying finite-dimensional filters.
However, beyond some very specific settings \cite{Daum05}, the optimal
filter is infinite-dimensional and impossible to compute in closed
form, requiring either approximate analytic techniques (e.g., the
extended Kalman filter (e.g., \cite{AndMoo05}), the unscented filter
\cite{JulUhl00}) or numerical procedures (e.g., particle filters
\cite{DouJoh09}). The latter usually require time discretization
and a finite number of particles, resulting in loss of precision .
For many practical tasks (e.g., queuing \cite{Bremaud81} and optical
communication \cite{SnyMil91}) and biologically motivated problems
(e.g., \cite{DayAbb05}) a natural observation process is given by
a point process observer, leading to a nonlinear infinite-dimensional
optimal filter (except in specific settings, e.g., finite state spaces,
\cite{Bremaud81,BobMeiEld09}). 

We consider a continuous-state and continuous-time multivariate hidden
Markov process observed through a set of sensory neuron-like elements
characterized by multi-dimensional tuning functions, representing
the elements' average firing rate. The tuning function parameters
are characterized by a distribution allowing much flexibility. The
actual firing of each cell is random and is given by a Poisson process
with rate determined by the input and by the cell's tuning function.
Inferring the hidden state under such circumstances has been widely
studied within the Computational Neuroscience literature, mostly for
static stimuli. In the more challenging and practically important
dynamic setting, much work has been devoted to the development of
numerical sampling techniques for fast and effective approximation
of the posterior distribution (e.g., \cite{Ahmadian2011}). In this
work we are less concerned with algorithmic issues, and more with
establishing closed-form analytic expressions for an approximately
optimal filter (see \cite{BobMeiEld09,SusMeiOpp11,SusMeiOpp13} for
previous work in related, but more restrictive settings), and using
these to characterize the nature of near-optimal encoders, namely
determining the structure of the tuning functions for optimal state
inference. A significant advantage of the closed form expressions
over purely numerical techniques is the insight and intuition that
is gained from them about \textit{qualitative} aspects of the system.
Moreover, the leverage gained by the analytic computation contributes
to reducing the variance inherent to Monte Carlo approaches. Technically,
given the intractable infinite-dimensional nature of the posterior
distribution, we use a projection method replacing the full posterior
at each point in time by a projection onto a simple family of distributions
(Gaussian in our case). This approach, originally developed in the
Filtering literature \cite{Maybeck79,BriHanLeg98}, and termed Assumed
Density Filtering (ADF), has been successfully used more recently
in Machine Learning \cite{Opper98,Minka01}. \textcolor{black}{As
far as we are aware, this is the first application of this methodology
to point process filtering. }

The main contributions of the paper are the following: (i) Derivation
of closed form recursive expressions for the continuous time posterior
mean and variance within the ADF approximation, allowing for the incorporation
of distributional assumptions over sensory variables. (ii) Characterization
of the optimal tuning curves (encoders) for sensory cells in a more
general setting than hitherto considered. Specifically, we study the
optimal shift of tuning curve centers, providing an explanation for
observed experimental phenomena \cite{HarMcAlp04}. (iii) Demonstration
that absence of spikes is informative, and that, depending on the
relationship between the tuning curve distribution and the dynamic
process (the `prior'), may significantly improve the inference. This
issue has not been emphasized in previous studies focusing on homogeneous
populations. 

We note that most previous work in the field of neural encoding/decoding
has dealt with static observations and was based on the Fisher information,
which often leads to misleading qualitative results (e.g., \cite{Bethge2002,YaeMei10}).
Our results address the full dynamic setting in continuous time, and
provide results for the posterior variance, which is shown to yield
an excellent approximation of the posterior Mean Square Error (MSE).
Previous work addressing non-uniform distributions over tuning curve
parameters \cite{GanSim14} used static univariate observations and
was based on Fisher information rather than the MSE itself.

\section{Problem formulation}

\subsection{Dense Gaussian neural code}

We consider a dynamical system with state $X_{t}\in\mathbb{R}^{n}$,
observed through an observation process $N$ describing the firing
patterns of sensory neurons in response to the process $X$. The observed
process is a diffusion process obeying the Stochastic Differential
Equation (SDE)
\[
dX_{t}=A\left(X_{t}\right)dt+D\left(X_{t}\right)dW_{t},\quad\left(t\geq0\right)
\]
where $A\left(\cdot\right),D\left(\cdot\right)$ are arbitrary functions
and $W_{t}$ is standard Brownian motion. The initial condition $X_{0}$
is assumed to have a continuous distribution with a known density.
The observation process $N$ is a marked point process \cite{SnyMil91}
defined on $\left[0,\infty\right)\times\mathbb{R}^{m}$, meaning that
each point, representing the firing of a neuron, is identified by
its time $t\in\left[0,\infty\right)$, and a mark $\theta\in\mathbb{R}^{m}$.
In this work the mark is interpreted as a parameter of the firing
neuron, which we refer to as the neuron's \emph{preferred stimulus}.
Specifically, a neuron with parameter $\theta$ is taken to have firing
rate
\[
\lambda\left(x;\theta\right)=h\exp\left(-\frac{1}{2}\left\Vert Hx-\theta\right\Vert _{\Sigma_{\mathrm{tc}}^{-1}}^{2}\right),
\]
in response to state $x$, where $H\in\mathbb{R}^{m\times n}$ and
$\Sigma_{\mathrm{tc}}\in\mathbb{R}^{m\times m}$ , $m\le n$, are
fixed matrices, and the notation $\left\Vert y\right\Vert _{M}^{2}$
denotes $y^{T}My$. The choice of Gaussian form for $\lambda$ facilitates
analytic tractability. The inclusion of the matrix $H$ allows using
high-dimensional models where only some dimensions are observed, for
example when the full state includes velocities but only locations
are directly observable. We also define $N_{t}\triangleq N\left(\left[0,t\right)\times\mathbb{R}^{m}\right)$,
i.e., $N_{t}$ is the total number of points up to time $t$, regardless
of their location $\theta$, and denote by $\mathcal{N}_{t}$ the
\emph{sequence} of points up to time $t$ --- formally, the process
$N$ restricted to $\left[0,t\right)\times\mathbb{R}^{m}$. Following
\cite{Snyder1991}, we use the notation
\begin{equation}
\int_{a}^{b}\int_{U}f\left(t,\theta\right)N\left(dt\times d\theta\right)\triangleq\sum_{i}\boldsymbol{1}\left\{ t_{i}\in\left[a,b\right],\theta_{i}\in U\right\} f\left(t_{i},\theta_{i}\right),\label{eq:point-integral}
\end{equation}
for $U\subseteq\mathbb{R}^{m}$ and any function $f$, where $\left(t_{i},\theta_{i}\right)$
are respectively the time and mark of the $i$-th point of the process
$N$.

Consider a network with $M$ sensory neurons, having random preferred
stimuli $\boldsymbol{\theta}=\left\{ \theta_{i}\right\} {}_{i=1}^{M}$
that are drawn independently from a common distribution with probability
density $f\left(\theta\right)$, which we refer to as the \emph{population
density}. Positing a distribution for the preferred stimuli allows
us to obtain simple closed form solutions, and to optimize over distribution
parameters rather than over the higher-dimensional space of all the
$\theta_{i}$. The total rate of spikes with preferred stimuli in
a set $A\subset\mathbb{R}^{m}$, given $X_{t}=x$, is then $\lambda_{A}\left(x;\boldsymbol{\theta}\right)=h\sum_{i}1_{\left\{ \theta_{i}\in A\right\} }\exp\left(-\frac{1}{2}\left\Vert Hx-\theta_{i}\right\Vert _{\Sigma_{\mathrm{tc}}^{-1}}^{2}\right)$.
Averaging over $f\left(\theta\right)$, we have the expected rate
$\lambda_{A}\left(x\right)\triangleq\mathrm{E}\lambda_{A}\left(x;\boldsymbol{\theta}\right)=hM\int_{A}f\left(\theta\right)\exp\left(-\frac{1}{2}\left\Vert Hx-\theta\right\Vert _{\Sigma_{\mathrm{tc}}^{-1}}^{2}\right)d\theta.$
We now obtain an infinite neural network by considering the limit
$M\to\infty$ while holding $\lambda^{0}=hM$ fixed. In the limit
we have $\lambda_{A}\left(x;\boldsymbol{\theta}\right)\to\lambda_{A}\left(x\right)$,
so that the process $N$ has density
\begin{equation}
\lambda_{t}\left(\theta,X_{t}\right)=\lambda^{0}f\left(\theta\right)\exp\left(-\frac{1}{2}\left\Vert HX_{t}-\theta\right\Vert _{\Sigma_{\mathrm{tc}}^{-1}}^{2}\right),\label{eq:space-time-density}
\end{equation}
meaning that the expected number of points in a small rectangle $\left[t,t+dt\right]\times\prod_{i}\left[\theta_{i},\theta_{i}+d\theta_{i}\right]$,
conditioned on the history $X_{\left[0,t\right]}$, is $\lambda_{t}\left(\theta,X_{t}\right)dt\prod_{i}d\theta_{i}+o\left(dt,\left|d\theta\right|\right)$.
A finite network can be obtained as a special case by taking $f$
to be a sum of delta functions representing a discrete distribution.

For analytic tractability, we assume that $f\left(\theta\right)$
is Gaussian with center $c$ and covariance $\Sigma_{\mathrm{pop}}$,
namely $f\left(\theta\right)={\cal N}(\theta;c,\Sigma_{\mathrm{pop}})$.
We refer to $c$ as the \emph{population center.} Previous work \cite{RhoSny1977,YaeMei10,Susemihl2014}
considered the case where neurons' preferred stimuli uniformly cover
the space, obtained by removing the factor $f\left(\theta\right)$
from \eqref{eq:space-time-density}. Then, the total firing rate $\int\lambda_{t}\left(\theta,x\right)d\theta$
is independent of $x$, which simplifies the analysis, and leads to
a Gaussian posterior (see \cite{RhoSny1977}). We refer to the assumption
that $\int\lambda_{t}\left(\theta,x\right)d\theta$ is independent
of $x$ as \emph{uniform coding}. The uniform coding case may be obtained
from our model by taking the limit $\Sigma_{\mathrm{pop}}^{-1}\to0$
with $\lambda^{0}/\sqrt{\det\Sigma_{\mathrm{pop}}}$ held constant.

\subsection{Optimal encoding and decoding}

We consider the question of optimal encoding and decoding under the
above model. The process of neural decoding is assumed to compute
(exactly or approximately) the full posterior distribution of $X_{t}$
given $\mathcal{N}_{t}$. The problem of neural encoding is then to
choose the parameters $\phi=\left(c,\Sigma_{\mathrm{pop}},\Sigma_{\mathrm{tc}}\right)$,
which govern the statistics of the observation process $N$, given
a specific decoding scheme.

To quantify the performance of the encoding-decoding system, we summarize
the result of decoding using a single estimator $\hat{X}_{t}=\hat{X}_{t}\left(\mathcal{N}_{t}\right)$,
and define the Mean Square Error (MSE) as $\epsilon_{t}\triangleq\mathrm{trace}[(X_{t}-\hat{X}_{t})(X_{t}-\hat{X}_{t})^{T}]$.
We seek $\hat{X}_{t}$ and $\phi$ that solve $\min_{\phi}\lim_{t\to\infty}\min_{\hat{X}_{t}}\mathrm{E}\left[\epsilon_{t}\right]=\min_{\phi}\lim_{t\to\infty}\mathrm{E}[\min_{\hat{X}_{t}}\mathrm{E}[\epsilon_{t}|\mathcal{N}_{t}]]$.
The inner minimization problem in this equation is solved by the MSE-optimal
decoder, which is the posterior mean $\hat{X}_{t}=\mu_{t}\triangleq\mathrm{E}\left[X_{t}|\mathcal{N}_{t}\right]$.
The posterior mean may be computed from the full posterior obtained
by decoding. The outer minimization problem is solved by the optimal
encoder. In principle, the encoding/decoding problem can be solved
for any value of $t$. In order to assess performance it is convenient
to consider the steady-state limit $t\to\infty$ for the encoding
problem. 

Below, we find a closed form approximate solution to the decoding
problem for any $t$ using ADF. We then explore the problem of choosing
the steady-state optimal encoding parameters $\phi$ using Monte Carlo
simulations. Note that if decoding is exact, the problem of optimal
encoding becomes that of minimizing the expected posterior variance.

\section{Neural decoding}

\subsection{Exact filtering equations}

Let $P\left(\cdot,t\right)$ denote the posterior density of $X_{t}$
given $\mathcal{N}_{t}$, and $\mathrm{E}_{P}^{t}\left[\cdot\right]$
the posterior expectation given $\mathcal{N}_{t}$. The prior density
$P\left(\cdot,0\right)$ is assumed to be known.

The problem of filtering a diffusion process $X$ from a doubly stochastic
Poisson process driven by $X$ is formally solved in \cite{Snyder1972}.
The result is extended to marked point processes in \cite{RhoSny1977},
where the authors derive a stochastic PDE for the posterior density\footnote{The model considered in \cite{RhoSny1977} assumes linear dynamics
and uniform coding -- meaning that the total rate of $N_{t}$, namely
$\int_{\theta}\lambda_{t}\left(\theta,X_{t}\right)d\theta$, is independent
of $X_{t}$. However, these assumption are only relevant to establish
other proposition in that paper. The proof of equation \eqref{eq:rhodes-snyder}
still holds as is in our more general setting.},
\begin{equation}
dP\left(x,t\right)=\mathcal{L}^{*}P\left(x,t\right)dt+P\left(x,t\right)\int_{\theta\in\mathbb{R}^{m}}\frac{\lambda_{t}\left(\theta,x\right)-\hat{\lambda}_{t}\left(\theta\right)}{\hat{\lambda}_{t}\left(\theta\right)}\left(N\left(dt\times d\theta\right)-\hat{\lambda}_{t}\left(\theta\right)d\theta\,dt\right),\label{eq:rhodes-snyder}
\end{equation}
where the integral with respect to $N$ is interpreted as in \eqref{eq:point-integral},
$\mathcal{L}$ is the state's posterior infinitesimal generator (Kolmogorov's
backward operator), defined as $\mathcal{L}f\left(x\right)=\lim_{\Delta t\to0^{+}}\left(\mathrm{E}\left[f\left(X_{t+\Delta t}\right)|X_{t}=x\right]-f\left(x\right)\right)/\Delta t$,
$\mathcal{L}^{*}$ is $\mathcal{L}$'s adjoint operator (Kolmogorov's
forward operator), and $\hat{\lambda}_{t}\left(\theta\right)\triangleq\mathrm{E}_{P}^{t}\left[\lambda_{t}\left(\theta,X_{t}\right)\right]=\int P\left(x,t\right)\lambda_{t}\left(\theta,x\right)dx$.

The stochastic PDE \eqref{eq:rhodes-snyder} is usually intractable.
In \cite{RhoSny1977,Susemihl2014} the authors consider linear dynamics
with uniform coding and Gaussian prior. In this case, the posterior
is Gaussian, and \eqref{eq:rhodes-snyder} leads to closed form ODEs
for its moments. When the uniform coding assumption is violated, the
posterior is no longer Gaussian. Still, we can obtain exact equations
for the posterior moments, as follows. 

Let $\mu_{t}=\mathrm{E}_{P}^{t}X_{t},\tilde{X}_{t}=X_{t}-\mu_{t},\Sigma_{t}=\mathrm{E}_{P}^{t}[\tilde{X}_{t}\tilde{X}_{t}^{T}]$.
Using \eqref{eq:rhodes-snyder}, along with known results about the
form of the infinitesimal generator $\mathcal{L}_{t}$ for diffusion
processes (see appendix), the first two posterior moments can be shown
to obey the following equations between spikes (see \cite{SusMeiOpp13}):
\begin{eqnarray}
\frac{d\mu_{t}}{dt} & = & \mathrm{E}_{P}^{t}\left[A\left(X_{t}\right)\right]+\mathrm{E}_{P}^{t}\left[X_{t}\int\left(\hat{\lambda}_{t}\left(\theta\right)-\lambda_{t}\left(\theta,X_{t}\right)\right)d\theta\right]\nonumber \\
\frac{d\Sigma_{t}}{dt} & = & \mathrm{E}_{P}^{t}\left[A\left(X_{t}\right)\tilde{X}_{t}^{T}\right]+\mathrm{E}_{P}^{t}\left[\tilde{X}_{t}A\left(X_{t}\right)^{T}\right]+\mathrm{E}_{P}^{t}\left[D\left(X_{t}\right)D\left(X_{t}\right)^{T}\right]\nonumber \\
 &  & +\mathrm{E}_{P}^{t}\left[\tilde{X}_{t}\tilde{X}_{t}^{T}\int\left(\hat{\lambda}_{t}\left(\theta\right)-\lambda_{t}\left(\theta,X_{t}\right)\right)d\theta\right]\,.\label{eq:moments}
\end{eqnarray}

\subsection{ADF approximation}

While equations \eqref{eq:moments} are exact, they are not practical,
since they require computation of $\mathrm{E}_{P}^{t}\left[\cdot\right]$.
We now proceed to find an approximate closed form for \eqref{eq:moments}.
Here we present the main ideas of the derivation. \textcolor{black}{The
formulation presented here assumes, for simplicity, an open-loop setting
where the system is passively observed. It can be readily extended
to a closed-loop control-based setting, and is presented in this more
general framework in the appendix, including full details.}

To bring \eqref{eq:moments} to a closed form, we use ADF with an
assumed Gaussian density (see \cite{Opper98} for details). Conceptually,
this may be envisioned as integrating \eqref{eq:moments} while replacing
the distribution $P$ by its approximating Gaussian ``at each time
step''. Assuming the moments are known exactly, the Gaussian is obtained
by matching the first two moments of $P$ \cite{Opper98}. Note that
the solution of the resulting equations does not in general match
the first two moments of the exact solution, though it may approximate
it. 

Abusing notation, in the sequel we use $\mu_{t},\Sigma_{t}$ to refer
to the ADF approximation rather than to the exact values. Substituting
the normal distribution $\mathcal{N}(x;\mu_{t},\Sigma_{t})$ for $P(x,t)$
to compute the expectations involving $\lambda_{t}$ in \eqref{eq:moments},
and using \eqref{eq:space-time-density} and the Gaussian form of
$f(\theta)$, results in computable Gaussian integrals. Other terms
may also be computed in closed form if the function $A,D$ can be
expanded as power series. This computation yields approximate equations
for $\mu_{t},\Sigma_{t}$ between spikes. The updates at spike times
can similarly be computed in closed form either from \eqref{eq:rhodes-snyder}
or directly from a Bayesian update of the posterior (see appendix,
or e.g., \cite{SusMeiOpp13}).

For simplicity, we assume that the dynamics are linear, $dX_{t}=AX_{t}+DdW_{t}$,
resulting in the filtering equations 
\begin{align}
d\mu_{t} & =A\mu_{t}dt+g_{t}\Sigma_{t}H^{T}S_{t}\left(H\mu_{t}-c\right)dt+\Sigma_{t^{-}}H^{T}S_{t^{-}}^{\mathrm{tc}}\int_{\theta\in\mathbb{R}^{m}}\left(\theta-H\mu_{t^{-}}\right)N\left(dt\times d\theta\right)\label{eq:adf-mu}\\
d\Sigma_{t} & =\left(A\Sigma_{t}+\Sigma_{t}A+DD^{T}\right)dt\nonumber \\
 & \quad+g_{t}\left[\Sigma_{t}H^{T}S_{t}H\Sigma_{t}-\Sigma_{t}H^{T}S_{t}\left(H\mu_{t}-c\right)\left(H\mu_{t}-c\right)^{T}S_{t}H\Sigma_{t}\right]dt\nonumber \\
 & \quad-\Sigma_{t^{-}}H^{T}S_{t^{-}}^{\mathrm{tc}}H\Sigma_{t^{-}}dN_{t},\label{eq:adf-sigma}
\end{align}
where $S_{t}^{\mathrm{tc}}\triangleq\left(\Sigma_{\mathrm{tc}}+H\Sigma_{t}H^{T}\right)^{-1},\,S_{t}\triangleq\left(\Sigma_{\mathrm{tc}}+\Sigma_{\mathrm{pop}}+H\Sigma_{t}H^{T}\right)^{-1}$,
and
\begin{align*}
g_{t} & \triangleq\int\hat{\lambda}\left(\theta\right)d\theta=\int\mathrm{E}_{P}^{t}\left[\lambda\left(\theta,X_{t}\right)\right]d\theta=\lambda^{0}\sqrt{\det\left(\Sigma_{\mathrm{tc}}S_{t}\right)}\exp\left(-\frac{1}{2}\left\Vert H\mu_{t}-c\right\Vert _{S_{t}}^{2}\right)
\end{align*}
is the posterior expected total firing rate. Expressions including
$t^{-}$ are to be interpreted as left limits $f\left(t^{-}\right)=\lim_{s\to t^{-}}f\left(s\right)$,
which are necessary since the solution is discontinuous at spike times.

The last term in \eqref{eq:adf-mu} is to be interpreted as in \eqref{eq:point-integral}.
It contributes an instantaneous jump in $\mu_{t}$ at the time of
a spike with preferred stimulus $\theta$, moving $H\mu_{t}$ closer
to $\theta$. Similarly, the last term in \eqref{eq:adf-sigma} contributes
an instantaneous jump in $\Sigma_{t}$ at each spike time, which is
the same regardless of spike location. All other terms describe the
evolution of the posterior between spikes: the first few terms in
\eqref{eq:adf-mu}-\eqref{eq:adf-sigma} are the same as in the dynamics
of the prior, as in \cite{SusMeiOpp13,Susemihl2014}, whereas the
terms involving $g_{t}$ correspond to information from the absence
of spikes. Note that the latter scale with $g_{t}$, the expected
total firing rate, i.e., lack of spikes becomes ``more informative''
the higher the expected rate of spikes.

It is illustrative to consider these equations in the scalar case
$m=n=1$, with $H=1$. Letting $\sigma_{t}^{2}=\Sigma_{t},\sigma_{\mathrm{tc}}^{2}=\Sigma_{\mathrm{tc}},\sigma_{\mathrm{pop}}^{2}=\Sigma_{\mathrm{pop}},a=A,d=D$
yields
\begin{align}
d\mu_{t} & =a\mu_{t}dt+g_{t}\frac{\sigma_{t}^{2}}{\sigma_{t}^{2}+\sigma_{\mathrm{tc}}^{2}+\sigma_{\mathrm{pop}}^{2}}\left(\mu_{t}-c\right)dt+\frac{\sigma_{t^{-}}^{2}}{\sigma_{t^{-}}^{2}+\sigma_{\mathrm{tc}}^{2}}\int_{\theta\in\mathbb{R}}\left(\theta-\mu_{t^{-}}\right)N\left(dt\times d\theta\right)\label{eq:adf-mu-1d}\\
d\sigma_{t}^{2} & =\left(2a\sigma_{t}^{2}+d^{2}+g_{t}\frac{\sigma_{t}^{2}}{\sigma_{t}^{2}+\sigma_{\mathrm{tc}}^{2}+\sigma_{\mathrm{pop}}^{2}}\left[1-\frac{\left(\mu_{t}-c\right)^{2}}{\sigma_{t}^{2}+\sigma_{\mathrm{tc}}^{2}+\sigma_{\mathrm{pop}}^{2}}\right]\sigma_{t}^{2}\right)dt-\frac{\sigma_{t^{-}}^{2}}{\sigma_{t^{-}}^{2}+\sigma_{\mathrm{tc}}^{2}}\sigma_{t^{-}}^{2}dN_{t},\label{eq:adf-sigma-1d}
\end{align}
where $g_{t}=\lambda^{0}\sqrt{2\pi\sigma_{\mathrm{tc}}^{2}}\mathcal{N}\left(\mu_{t};c,\sigma_{t}^{2}+\sigma_{\mathrm{tc}}^{2}+\sigma_{\mathrm{pop}}^{2}\right)$.
Figure \ref{1d-filtering} (left) shows how $\mu_{t},\sigma_{t}^{2}$
change between spikes for a static 1-dimensional state ($a=d=0$).
In this case, all terms in the filtering equations drop out except
those involving $g_{t}$. The term involving $g_{t}$ in $d\mu_{t}$
pushes $\mu_{t}$ away from $c$ in the absence of spikes. This effect
weakens as $\left|\mu_{t}-c\right|$ grows due to the factor $g_{t}$,
consistent with the idea that far from $c$, the lack of spikes is
less surprising, hence less informative. The term involving $g_{t}$
in $d\sigma_{t}^{2}$ increases the variance when $\mu_{t}$ is near
$c$, otherwise decreases it.

\begin{figure}
\includegraphics[bb=0bp 10bp 720bp 420bp,clip,height=0.15\textheight]{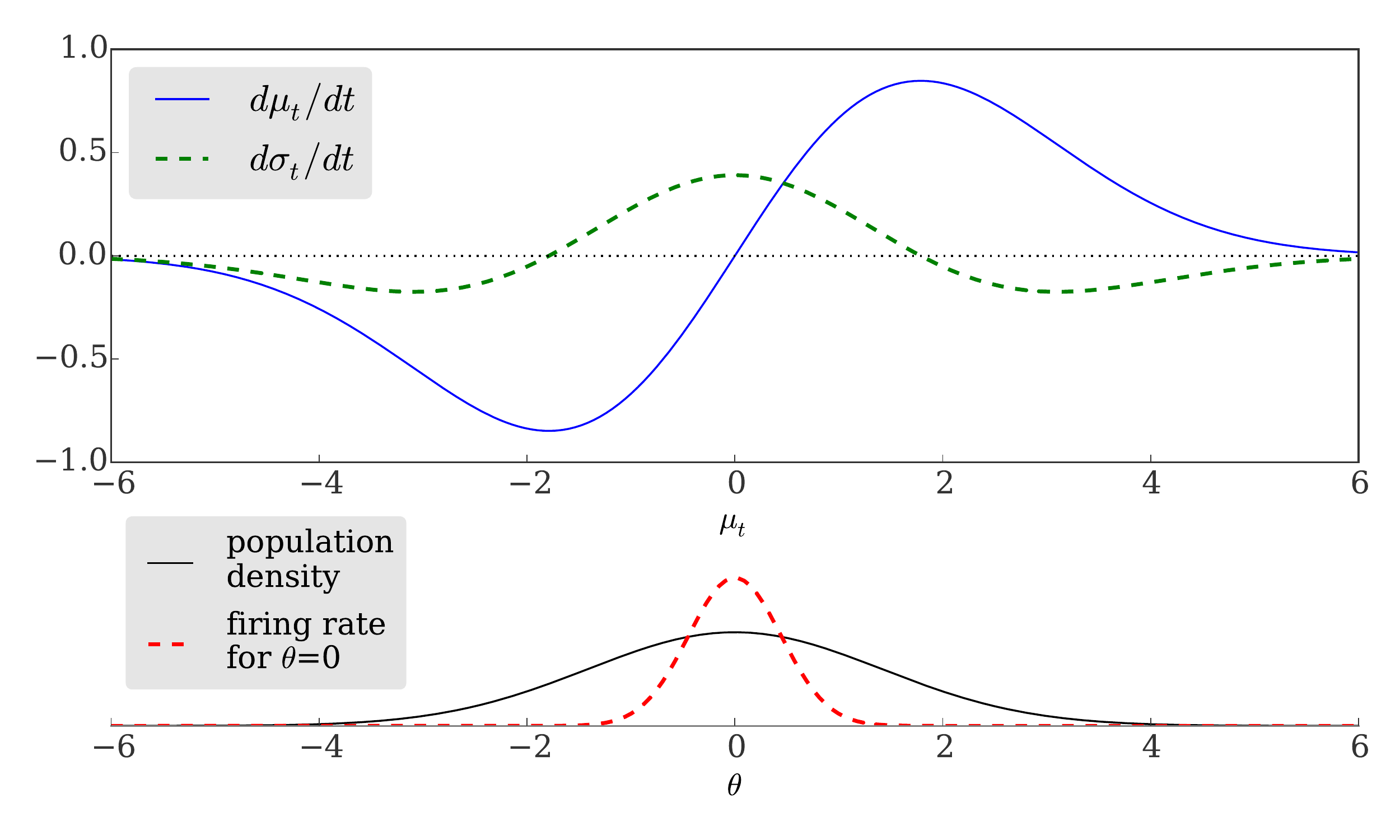}\hfill{}\includegraphics[bb=0bp 10bp 720bp 390bp,clip,height=0.15\textheight]{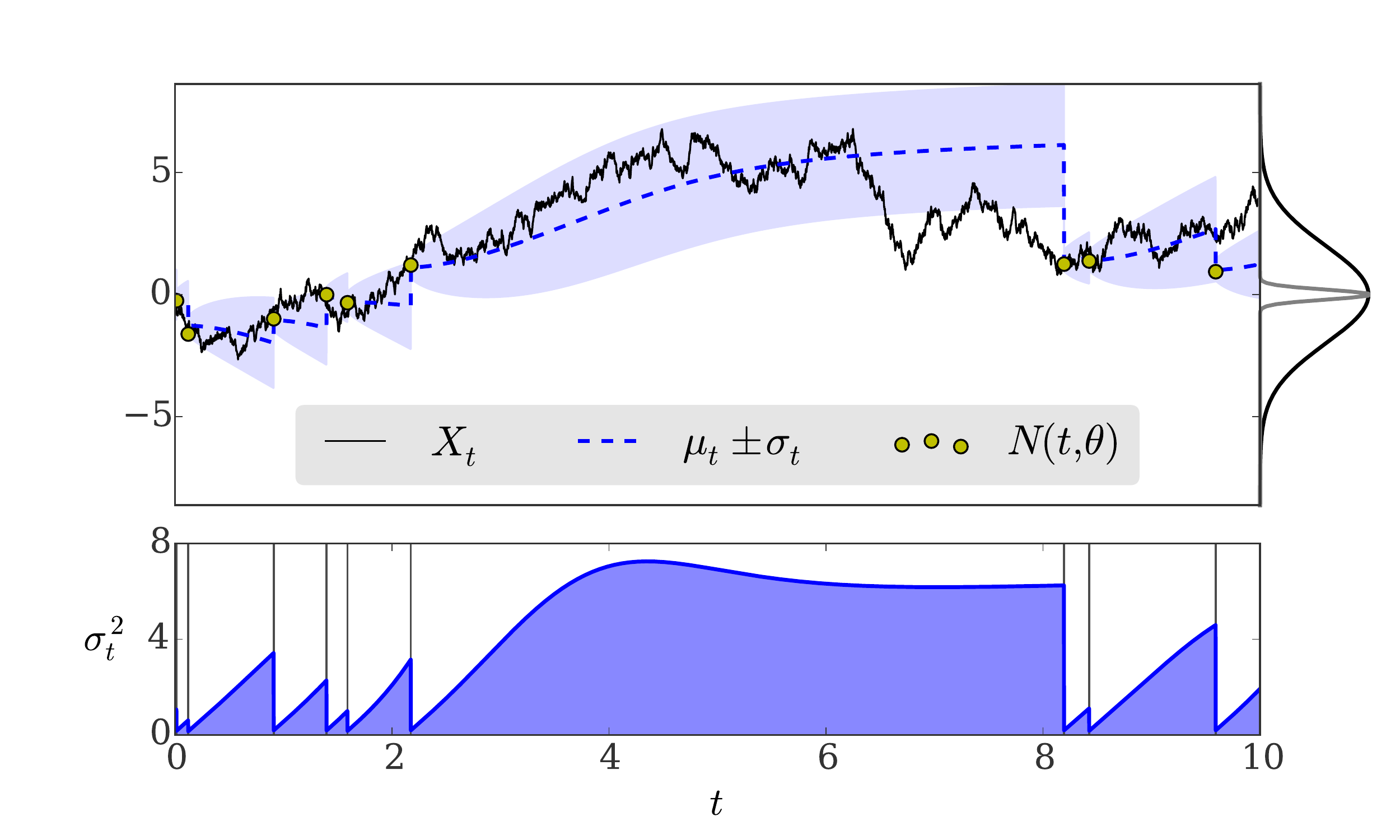}\protect\caption{\textbf{Left} Changes to the posterior moments between spikes as a
function of the current posterior mean estimate, for a static 1-d
state. The parameters are $a=d=0,H=1,\sigma_{\mathrm{pop}}^{2}=1,\sigma_{\mathrm{tc}}^{2}=0.2,c=0,\lambda^{0}=10,\sigma_{t}=1$.
The bottom plot shows the density of preferred stimuli $f\left(\theta\right)$
and tuning curve for a neuron with preferred stimulus $\theta=0$.
\textbf{Right} An example of filtering a linear one-dimensional process.
Each dot correspond to a spike with the vertical location indicating
the preferred stimulus $\theta$. The curves to the right of the graph
show the preferred stimulus density (black), and a tuning curve centered
at $\theta=0$ (gray). The tuning curve and preferred stimulus density
are normalized to the same height for visualization. The bottom graph
shows the posterior variance, with the vertical lines showing spike
times. Parameters are: $a=-0.1,d=2,H=1,\sigma_{\mathrm{pop}}^{2}=2,\sigma_{\mathrm{tc}}^{2}=0.2,c=0,\lambda^{0}=10,\mu_{0}=0,\sigma_{0}^{2}=1$.
Note the decrease of the posterior variance following $t=4$ even
though no spikes are observed.}
\label{1d-filtering}
\end{figure}

\subsection{Information from lack of spikes}

An interesting aspect of the filtering equations \eqref{eq:adf-mu}-\eqref{eq:adf-sigma}
is that the dynamics of the posterior density between spikes differ
from the prior dynamics. This is in contrast to previous models which
assumed uniform coding: the (exact) filtering equations appearing
in \cite{RhoSny1977} and \cite{Susemihl2014} have the same form
as \eqref{eq:adf-mu}-\eqref{eq:adf-sigma} except that they do not
include the correction terms involving $g_{t}$, so that between spikes
the dynamics are identical to the prior dynamics. This reflects the
fact that lack of spikes in a time interval is an indication that
the total firing rate is low; in the uniform coding case, this is
not informative, since the total firing rate is independent of the
state.

Figure \ref{adf-vs-uniform} (left) illustrates the information gained
from lack of spikes. A static scalar state is observed by a process
with rate \eqref{eq:space-time-density}, and filtered twice: once
with the correct value of $\sigma_{\mathrm{pop}}$, and once with
$\sigma_{\mathrm{pop}}\rightarrow\infty$, as in the uniform coding
filter of \cite{Susemihl2014}. Between spikes, the ADF estimate moves
away from the population center $c=0$, whereas the uniform coding
estimate remains fixed. The size of this effect decreases with time,
as the posterior variance estimate (not shown) decreases. The reduction
in filtering errors gained from the additional terms in \eqref{eq:adf-mu}-\eqref{eq:adf-sigma}
is illustrated in Figure \eqref{adf-vs-uniform} (right). Despite
the approximation involved, the full filter significantly outperforms
the uniform coding filter. The difference disappears as $\sigma_{\mathrm{pop}}$
increases and the population becomes uniform.

\begin{figure}
\includegraphics[bb=0bp 0bp 574bp 423bp,height=0.18\textheight]{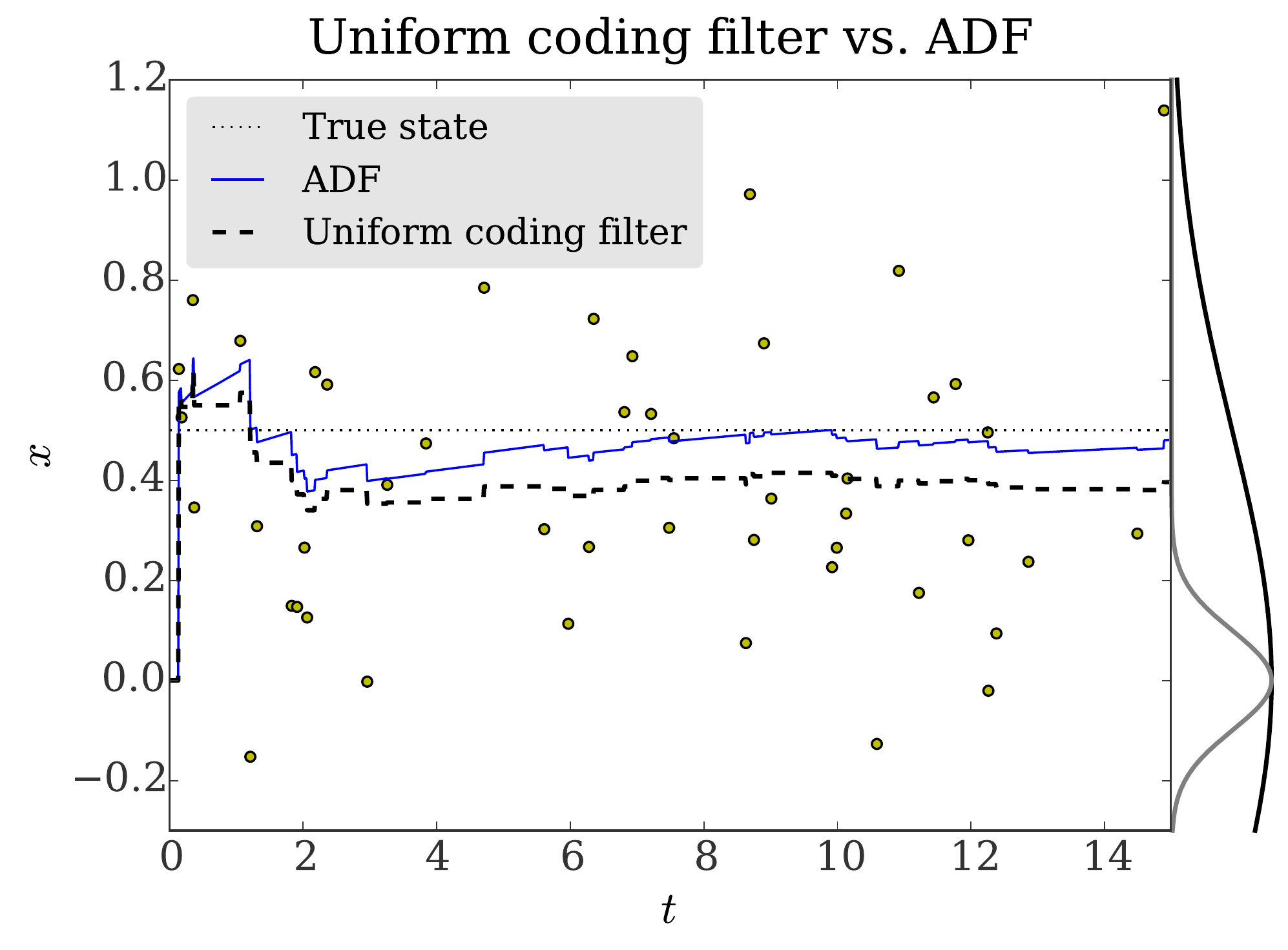}\hfill{}\includegraphics[bb=0bp 0bp 644bp 430bp,height=0.18\textheight]{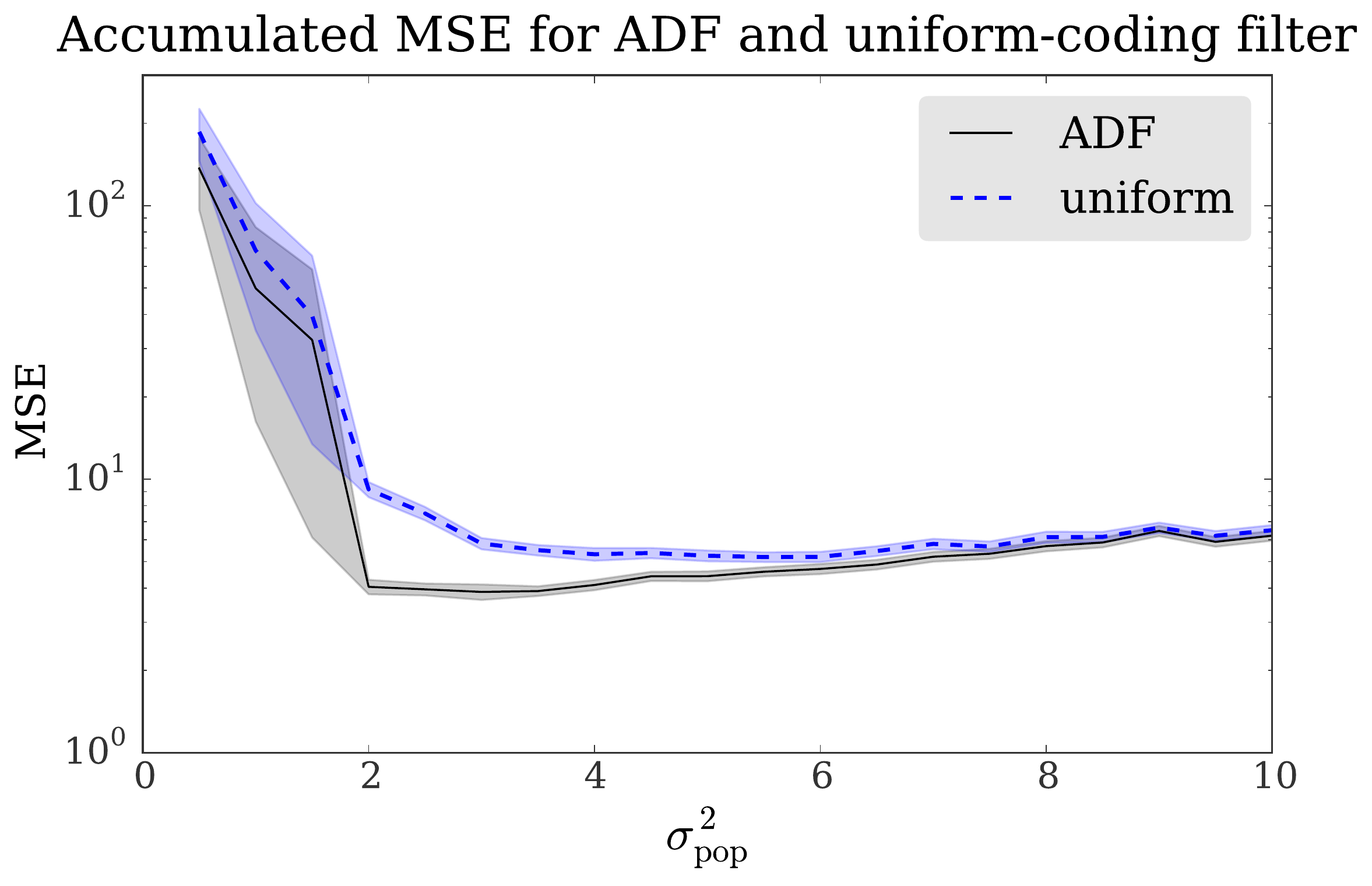}\protect\caption{\textbf{Left} Illustration of information gained between spikes. A
static state $X_{t}=0.5$, shown in a dotted line, is observed and
filtered twice: with the correct value $\sigma_{\mathrm{pop}}^{2}=0.5$
(``ADF'', solid blue line), and with $\sigma_{\mathrm{pop}}^{2}=\infty$
(``Uniform coding filter'', dashed green line). The curves to the
right of the graph show the preferred stimulus density (black), and
a tuning curve centered at $\theta=0$ (gray). Both filters are initialized
with $\mu_{0}=0,\sigma_{0}^{2}=1$. \textbf{Right} Comparison of MSE
for the ADF filter and the uniform coding filter. The vertical axis
shows the integral of the square error integrated over the time interval
$\left[5,10\right]$, averaged over 1000 trials. Shaded areas indicate
estimated errors, computed as the sample standard deviation divided
by the square root of the number of trials. Parameters in both plots
are $a=d=0,c=0,\sigma_{\mathrm{pop}}^{2}=0.5,\sigma_{\mathrm{tc}}^{2}=0.1,H=1,\lambda^{0}=10$.}
\label{adf-vs-uniform}
\end{figure}

\textbf{Special cases~~ }To gain additional insight into the filtering
equations, we consider their behavior in several limits. (i) As $\sigma_{\mathrm{pop}}^{2}\to\infty$,
spikes become rare as the density $f\left(\theta\right)$ approaches
0 for any $\theta$. The total expected rate of spikes $g_{t}$ also
approaches 0, and the terms corresponding to information from lack
of spikes vanish. Other terms in the equations are unaffected. (ii)
In the limit $\sigma_{\mathrm{tc}}^{2}\to\infty$, each neuron fires
as a Poisson process with a constant rate independent of the observed
state. The total expected firing rate $g_{t}$ saturates at its maximum,
$\lambda^{0}$. Therefore the preferred stimuli of spiking neurons
provide no information, nor does the presence or absence of spikes.
Accordingly, all terms other than those related to the prior dynamics
vanish. (iii) The uniform coding case \cite{RhoSny1977,Susemihl2014}
is obtained as a special case in the limit $\sigma_{\mathrm{pop}}^{2}\to\infty$
with $\lambda^{0}/\sigma_{\mathrm{pop}}$ constant. In this limit
the terms involving $g_{t}$ drop out, recovering the (exact) filtering
equations in \cite{RhoSny1977}.

\section{Optimal neural encoding}

We model the problem of optimal neural encoding as choosing the parameters
$c,\Sigma_{\mathrm{pop}},\Sigma_{\mathrm{tc}}$ of the population
and tuning curves, so as to minimize the steady-state MSE. As noted
above, when the estimate is exactly the posterior mean, this is equivalent
to minimizing the steady-state expected posterior variance. The posterior
variance has the advantage of being less noisy than the square error
itself, since by definition it is the mean of the square error (of
the posterior mean) under conditioning by $\mathcal{N}_{t}$. We explore
the question of optimal neural encoding by measuring the steady-state
variance through Monte Carlo simulations of the system dynamics and
the filtering equations \eqref{eq:adf-mu}-\eqref{eq:adf-sigma}.
Since the posterior mean and variance computed by ADF are approximate,
we verified numerically that the variance closely matches the MSE
in the steady state when averaged across many trials (see appendix),
suggesting that asymptotically the error in estimating $\mu_{t}$
and $\Sigma_{t}$ is small.

\subsection{Optimal population center}

We now consider the question of the optimal value for the population
center $c$. Intuitively, if the prior distribution of the process
$X$ is unimodal with mode $x_{0}$, the optimal population center
is at $Hx_{0}$, to produce the most spikes. On the other hand, the
terms involving $g_{t}$ in the filtering equation \eqref{eq:adf-mu}-\eqref{eq:adf-sigma}
suggest that the lack of spikes is also informative. Moreover, as
seen in Figure \ref{1d-filtering} (left), the posterior variance
is reduced between spikes only when the current estimate is far enough
from $c$. These considerations suggest that there is a trade-off
between maximizing the frequency of spikes and maximizing the information
obtained from lack of spikes, yielding an optimal value for $c$ that
differs from $Hx_{0}$.

We simulated a simple one-dimensional process to determine the optimal
value of $c$ which minimizes the approximate posterior variance $\Sigma_{t}$.
Figure \ref{optimal-c} (left) shows the posterior variance for varying
values of the population center $c$ and base firing rate $\lambda^{0}$.
For each firing rate, we note the value of $c$ minimizing the posterior
variance (the optimal population center), as well as the value of
$c_{\mathrm{m}}=\mathrm{argmin}_{c}\left(d\sigma_{t}/dt|_{\mu_{t}=0}\right)$,
which maximizes the reduction in the posterior variance when the current
state estimate $\mu_{t}$ is at the process equilibrium $x_{0}=0$.
Consistent with the discussion above, the optimal value lies between
0 (where spikes are most abundant) and $c_{\mathrm{m}}$ (where lack
of spikes is most informative). As could be expected, the optimal
center is closer to 0 the higher the base firing rate. Similarly,
wide tuning curves, which render the spikes less informative, lead
to an optimal center farther from 0 (Figure \ref{optimal-c}, right).

A shift of the population center relative to the prior mode has been
observed physiologically in encoding of inter-aural time differences
for localization of sound sources \cite{Brand2002}. In \cite{HarMcAlp04},
this phenomenon was explained in a finite population model based on
maximization of Fisher information. This is in contrast to the results
of \cite{GanSim14}, which consider a heterogeneous population where
the tuning curve width scales roughly inversely with neuron density.
In this case, the population density maximizing the Fisher information
is shown to be monotonic with the prior, i.e., more neurons should
be assigned to more probable states. This apparent discrepancy may
be due to the scaling of tuning curve widths in \cite{GanSim14},
which produces roughly constant total firing rate, i.e., uniform coding.
This demonstrates that a non-constant total firing rate, which renders
lack of spikes informative, may be necessary to explain the physiologically
observed shift phenomenon.

\begin{figure}
\includegraphics[bb=0bp 10bp 699bp 267bp,clip,width=0.6\textwidth]{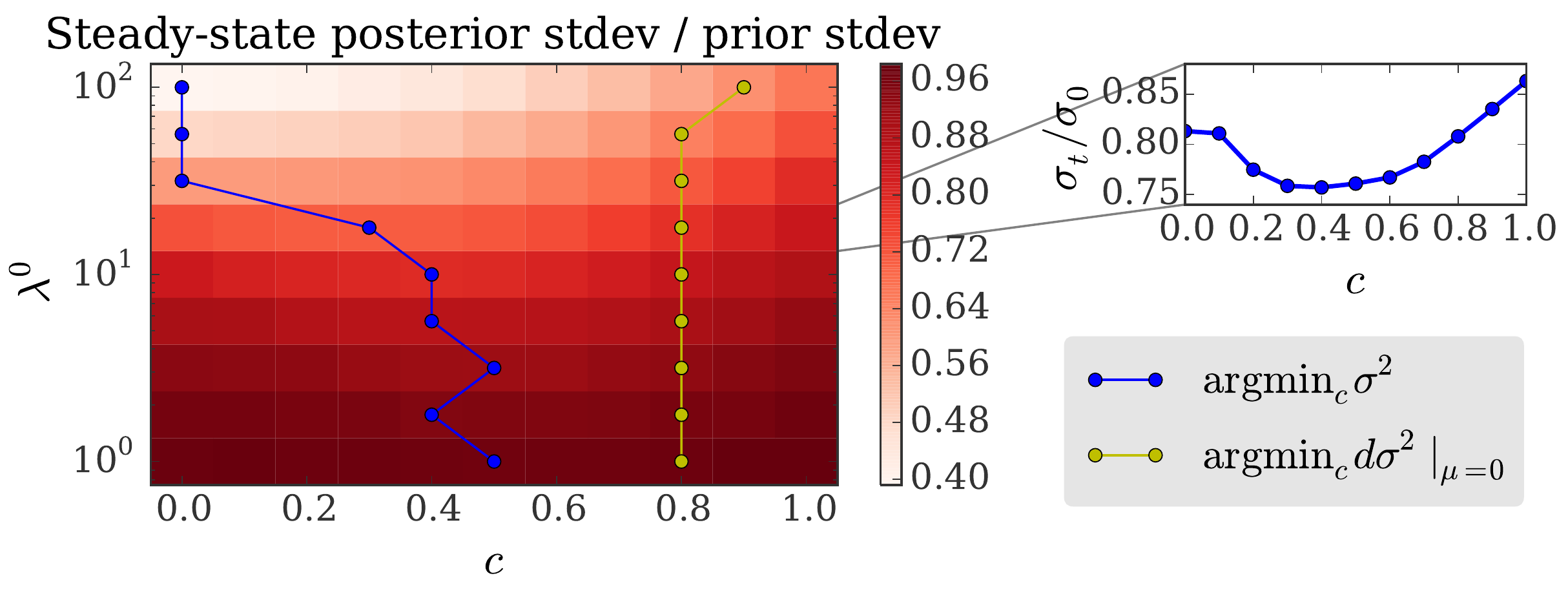}\hfill{}\includegraphics[bb=0bp 10bp 447bp 303bp,clip,width=0.35\textwidth]{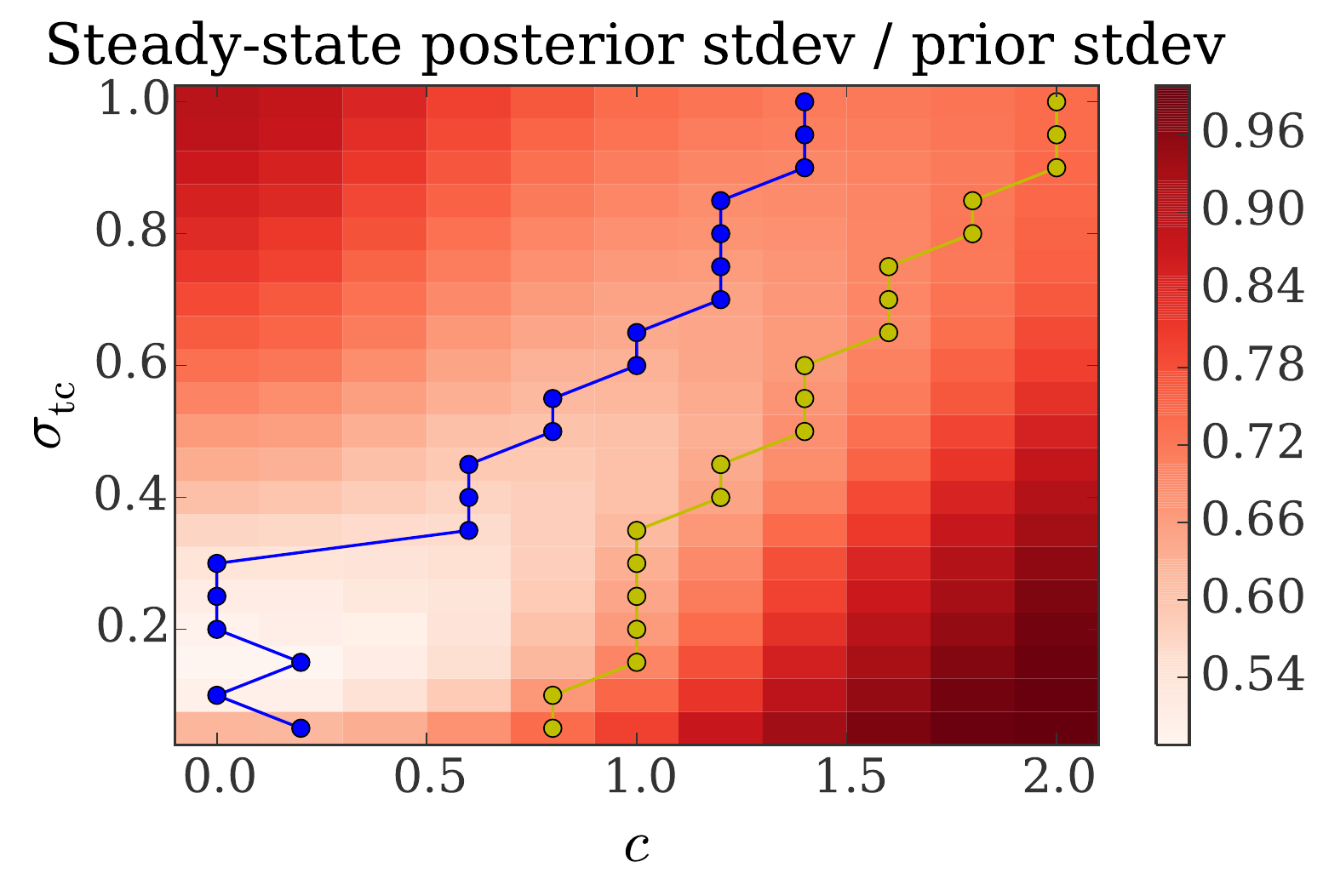}\ ~~\protect\caption{Optimal population center location for filtering a linear one-dimensional
process. Both graphs show the ratio of posterior standard deviation
to the prior steady-state standard deviation of the process, along
with the value of $c$ minimizing the posterior variance (blue line),
and minimizing the reduction of posterior variance when $\mu_{t}=0$
(yellow line). The process is initialized from its steady-state distribution.
The posterior variance is estimated by averaging over the time interval
$\left[5,10\right]$ and across 1000 trials for each data point. Parameters
for both graphs: $a=-1,d=0.5,\sigma_{\mathrm{pop}}^{2}=0.1$. In the
graph on the left, $\sigma_{\mathrm{tc}}^{2}=0.01$; on the right,
$\lambda^{0}=50$.}
\label{optimal-c}
\end{figure}

\subsection{Optimization of population distribution}

Next, we consider the optimization of the population distribution,
namely, the simultaneous optimization of the population center $c$
and the population variance $\Sigma_{\mathrm{pop}}$ in the case of
a static scalar state. Previous work using a finite neuron population
and a Fisher information-based criterion \cite{HarMcAlp04} has shown
that the optimal distribution of preferred stimuli depends on the
prior variance. When it is small relative to the tuning curve width,
optimal encoding is achieved by placing all preferred stimuli at a
fixed distance from the prior mean. On the other hand, when the prior
variance is large relative to the tuning curve width, optimal encoding
is uniform (see figure 2 in \cite{HarMcAlp04}).

Similar results are obtained with our model, as shown in Figure \ref{harper}.
Here, a static scalar state drawn from $\mathcal{N}(0,\sigma_{\mathrm{p}}^{2})$
is filtered by a population with tuning curve width $\sigma_{\mathrm{tc}}=1$
and preferred stimulus density $\mathcal{N}(c,\sigma_{\mathrm{pop}}^{2})$.
In Figure \ref{harper} (left), the prior distribution is narrow relative
to the tuning curve width, leading to an optimal population with a
narrow population distribution far from the origin. In Figure \ref{harper}
(right), the prior is wide relative to the tuning curve width, leading
to an optimal population with variance that roughly matches the prior
variance. When both the tuning curves and the population density are
narrow relative to the prior, so that spikes are rare (low values
of $\sigma_{\mathrm{pop}}$ in Figure \ref{harper} (right)), the
ADF approximation becomes poor, resulting in MSEs larger than the
prior variance.

\begin{figure}
\includegraphics[bb=0bp 10bp 589bp 305bp,clip,width=0.45\textwidth]{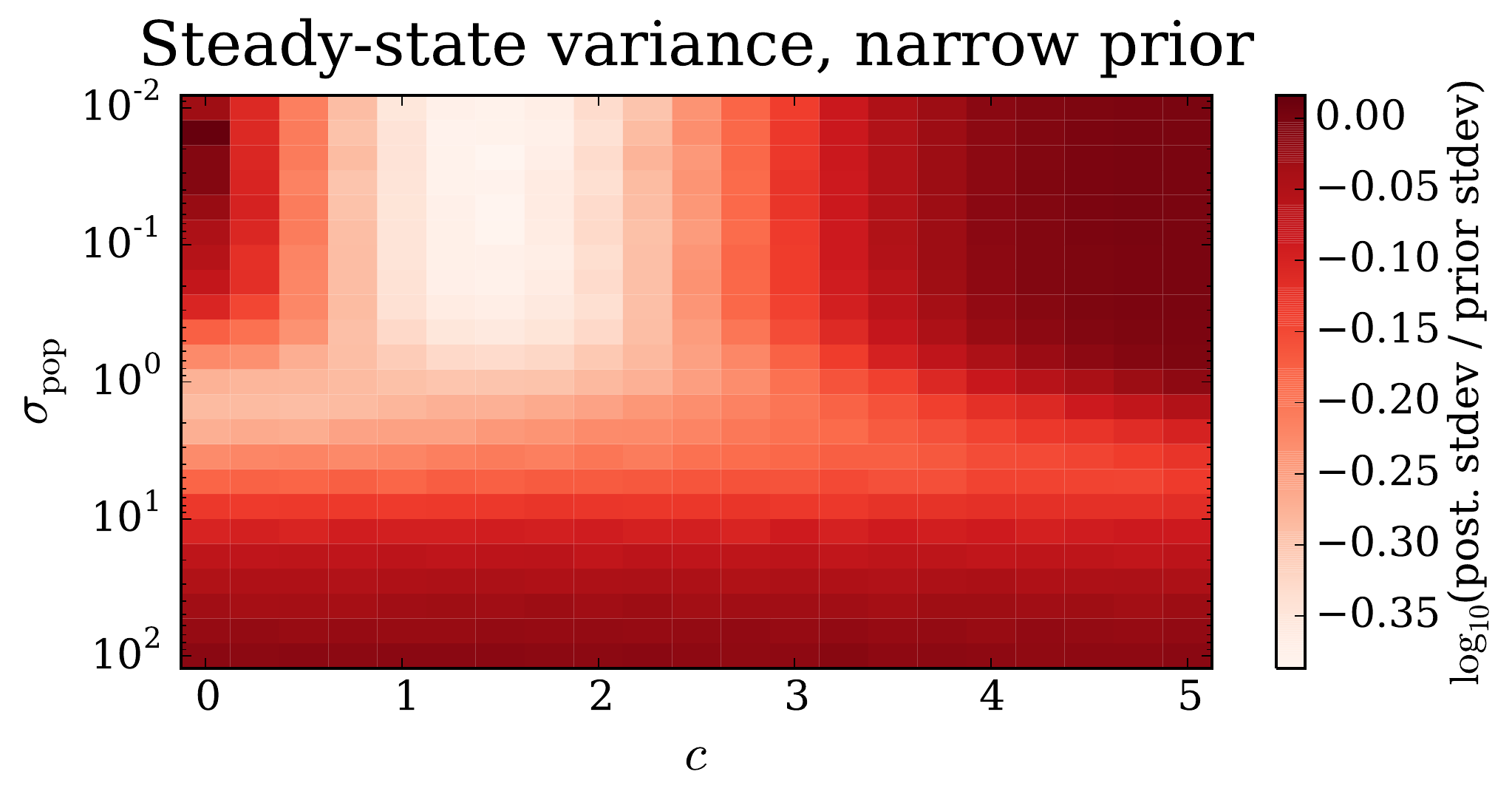}\hfill{}\includegraphics[width=0.4\textwidth]{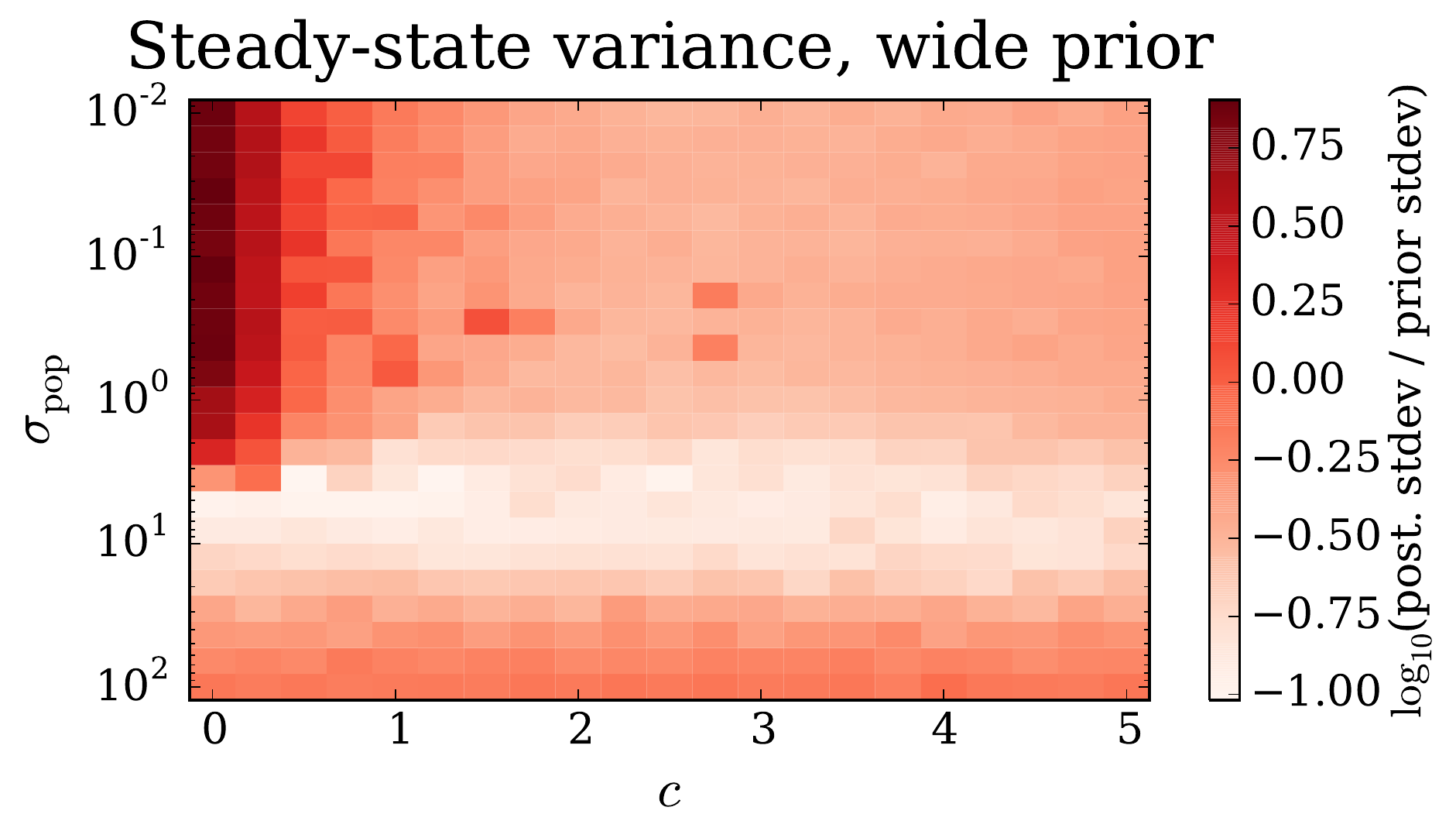}\protect\caption{Optimal population distribution depends on prior variance relative
to tuning curve width. A static scalar state drawn from $\mathcal{N}(0,\sigma_{\mathrm{p}}^{2})$
is filtered with tuning curve $\sigma_{\mathrm{tc}}=1$ and preferred
stimulus density $\mathcal{N}(c,\sigma_{\mathrm{pop}}^{2})$. Both
graphs show the posterior standard deviation relative to the prior
standard deviation $\sigma_{\mathrm{p}}$. In the left graph, the
prior distribution is narrow, $\sigma_{\mathrm{p}}^{2}=0.1$, whereas
on the right, it is wide, $\sigma_{\mathrm{p}}^{2}=10$. In both cases
the filter is initialized with the correct prior, and the square error
is averaged over the time interval $\left[5,10\right]$ and across
100 trials for each data point.}
\label{harper}
\end{figure}

\section{Conclusions}

We have introduced an analytically tractable Bayesian approximation
to point process filtering, allowing us to gain insight into the generally
intractable infinite-dimensional filtering problem. The approach enables
the derivation of near-optimal encoding schemes going beyond previously
studied uniform coding assumptions. The framework is presented in
continuous time, circumventing temporal discretization errors and
numerical imprecisions in sampling-based methods, applies to fully
dynamic setups, and directly estimates the MSE rather than lower bounds
to it. It successfully explains observed experimental results, and
opens the door to many future predictions. Future work will include
a development of previously successful mean field approaches \cite{SusMeiOpp13}
within our more general framework, leading to further analytic insight.
Moreover, the proposed strategy may lead to practically useful decoding
of spike trains. 

\appendix

\section{Appendix}

\subsection{Derivation of the ADF filtering equations for linear dynamics}

\subsubsection{Setting and notation}

In the main text, we have presented our model in an open-loop setting,
where the process $X$ is passively observed. Here we consider a more
general setting, incorporating a control process $U_{t}$, so the
dynamics are
\begin{equation}
dX_{t}=\left(A\left(X_{t}\right)+B\left(U_{t}\right)\right)dt+D\left(X_{t}\right)dW_{t},\label{eq:closed_loop_dyn}
\end{equation}
where, in general, $U_{t}$ is a function of $\mathcal{N}_{t}$. 

For the purposes of the derivation, it is convenient to work with
precision matrices rather than variance matrices. We write $F=\Sigma_{\mathrm{pop}}^{-1},R=\Sigma_{\mathrm{tc}}^{-1}$
and $Q_{t}=\Sigma_{t}^{-1}$. Thus the density of the process $N$
at $\left(t,\theta\right)$ given $X_{\left[0,t\right]}$ is
\begin{align}
\lambda_{t}\left(\theta,X_{t}\right) & =\lambda^{0}\sqrt{\frac{\left|F\right|}{\left(2\pi\right)^{m}}}\exp\left(-\frac{1}{2}\left\Vert \theta-c\right\Vert _{F}^{2}-\frac{1}{2}\left\Vert HX_{t}-\theta\right\Vert _{R}^{2}\right),\label{eq:space-time-density-precision}\\
 & =\lambda^{0}\sqrt{\frac{\left|F\right|}{\left(2\pi\right)^{m}}}\exp\left(-\frac{1}{2}\left\Vert HX_{t}-c\right\Vert _{M}^{2}-\frac{1}{2}\left\Vert \theta-\left(F+R\right)^{-1}\left(Fc+RHX_{t}\right)\right\Vert _{F+R}^{2}\right),
\end{align}
where $M\triangleq\left(F^{-1}+R^{-1}\right)^{-1}$.

We denote by $P\left(\cdot,t\right)$ the posterior density of $X_{t}$
given $\mathcal{N}_{t}$, and by $\mathrm{E}_{P}^{t}\left[\cdot\right]$
the posterior expectation based on observations up to time $t$. We
will simply write $\mathrm{E}_{P}\left[\cdot\right]$ when the time
$t$ is obvious from context.

\subsubsection{Filtering equations between spikes}

\paragraph{Exact filtering equations for the first two moments}

As seen in \cite{RhoSny1977}, the PDE for the posterior density,
\begin{equation}
d_{t}P\left(x,t\right)=\mathcal{L}_{t}^{*}P\left(x,t\right)dt+P\left(x,t\right)\int_{\mathbb{R}^{n}}\frac{\lambda_{t}\left(\theta,x\right)-\hat{\lambda}_{t}\left(\theta\right)}{\hat{\lambda}_{t}\left(\theta\right)}\left(N\left(dt\times d\theta\right)-\hat{\lambda}_{t}\left(\theta\right)d\theta\,dt\right),\label{eq:rhodes-snyder-time-variant}
\end{equation}
still holds in the closed-loop case. Here $\mathcal{L}_{t}$ (appearing
in place of $\mathcal{L}$ in \eqref{eq:rhodes-snyder}) is the posterior
infinitesimal generator, defined with an additional conditioning on
$\mathcal{N}_{t}$,
\[
\mathcal{L}_{t}f\left(x\right)=\lim_{\Delta t\to0^{+}}\frac{\mathrm{E}\left[f\left(X_{t+\Delta t}\right)|X_{t}=x,\mathcal{N}_{t}\right]-f\left(x\right)}{\Delta t},
\]
and $\mathcal{L}_{t}^{*}$ is its adjoint. Note that in this closed-loop
setting, the infinitesimal generator is itself a random operator,
due to its dependence on past observations through the control law,
and that $N_{t}$ is no longer a doubly-stochastic Poisson process.

\textit{Between spikes}, \eqref{eq:rhodes-snyder-time-variant} simplifies
to 
\[
\frac{\partial}{\partial t}P\left(x,t\right)=\mathcal{L}_{t}^{*}P\left(x,t\right)-P\left(x,t\right)\int_{\mathbb{R}^{n}}\left(\lambda_{t}\left(\theta,x\right)-\hat{\lambda}_{t}\left(\theta\right)\right)d\theta,
\]
so for a sufficiently well-behaved function $f$,
\begin{eqnarray*}
\frac{\partial\mathrm{E}_{P}\left[f\left(X_{t}\right)\right]}{\partial t} & = & \int f\left(x\right)\left(\mathcal{L}_{t}^{*}P\left(x,t\right)+P\left(x,t\right)\int\left(\hat{\lambda}_{t}\left(\theta\right)-\lambda_{t}\left(\theta,x\right)\right)d\theta\right)dx\\
 & = & \int P\left(x,t\right)\left(\mathcal{L}_{t}f\left(x\right)+f\left(x\right)\int\left(\hat{\lambda}_{t}\left(\theta\right)-\lambda_{t}\left(\theta,x\right)\right)d\theta\right)dx\\
 & = & \mathrm{E}_{P}\left[\mathcal{L}_{t}f\left(X_{t}\right)+f\left(X_{t}\right)\int\left(\hat{\lambda}_{t}\left(\theta\right)-\lambda_{t}\left(\theta,X_{t}\right)\right)d\theta\right].
\end{eqnarray*}

Assuming the state evolves as in \eqref{eq:closed_loop_dyn}, the
(closed loop) infinitesimal generator is
\[
\mathcal{L}_{t}f\left(x\right)=\left(A\left(x\right)+B\left(U_{t}\right)\right)^{T}\nabla f\left(x\right)+\frac{1}{2}\mathrm{Tr}\left[\nabla^{2}f\left(x\right)D\left(x\right)D\left(x\right)^{T}\right],
\]
so, letting $\mu_{t}=\mathrm{E}_{P}X_{t},\tilde{X}_{t}=X_{t}-\mu_{t},\Sigma_{t}=\mathrm{E}_{P}\left[\tilde{X}_{t}\tilde{X}_{t}^{T}\right]$,
\begin{eqnarray}
\frac{d\mu_{t}}{dt} & = & \mathrm{E}_{P}\left[A\left(X_{t}\right)\right]+B\left(U_{t}\right)+\mathrm{E}_{P}\left[X_{t}\int\left(\hat{\lambda}_{t}\left(\theta\right)-\lambda_{t}\left(\theta,X_{t}\right)\right)d\theta\right],\nonumber \\
\frac{d\Sigma_{t}}{dt} & = & \mathrm{E}_{P}\left[A\left(X_{t}\right)\tilde{X}_{t}^{T}\right]+\mathrm{E}_{P}\left[\tilde{X}_{t}A\left(X_{t}\right)^{T}\right]+\mathrm{E}_{P}\left[D\left(X_{t}\right)D\left(X_{t}\right)^{T}\right]\nonumber \\
 &  & +\mathrm{E}_{P}\left[\tilde{X}_{t}\tilde{X}_{t}^{T}\int\left(\hat{\lambda}_{t}\left(\theta\right)-\lambda_{t}\left(\theta,X_{t}\right)\right)d\theta\right],\label{eq:moments-closed-loop}
\end{eqnarray}
which is the same as \eqref{eq:moments}, with an additional term
$B\left(U_{t}\right)$.

\paragraph{ADF approximation}

The computations that follow frequently require multiplying Gaussian
functions, sometimes with a possibly degenerate precision matrix.
To this end, we use the following slightly generalized form of a well-known
result about the sum of quadratic forms.
\begin{claim*}
Let $x,a,b\in\mathbb{R}^{n}$ and $A,B\in\mathbb{R}^{n\times n}$
be symmetric matrices such that $A+B$ is non-singular. Then
\[
\left\Vert x-a\right\Vert _{A}^{2}+\left\Vert x-b\right\Vert _{B}^{2}=\left\Vert a-b\right\Vert _{A\left(A+B\right)^{-1}B}^{2}+\left\Vert x-\left(A+B\right)^{-1}\left(Aa+Bb\right)\right\Vert _{A+B}^{2}.
\]
\end{claim*}
\begin{proof}
This is proved by a straightforward completion of squares. If $A,B$
are invertible, 
\begin{eqnarray*}
\left\Vert x-a\right\Vert _{A}^{2}+\left\Vert x-b\right\Vert _{B}^{2} & = & \left\Vert x\right\Vert _{A}^{2}-x^{T}Aa-a^{T}Ax+\left\Vert a\right\Vert _{A}^{2}+\left\Vert x\right\Vert _{B}^{2}-x^{T}Bb-b^{T}Bx+\left\Vert b\right\Vert _{B}^{2}\\
 & = & \left\Vert x\right\Vert _{A+B}^{2}-x^{T}\left(Aa+Bb\right)-\left(Aa+Bb\right)^{T}x+\left\Vert a\right\Vert _{A}^{2}+\left\Vert b\right\Vert _{B}^{2}\\
 & = & \left\Vert x-\left(A+B\right)^{-1}\left(Aa+Bb\right)\right\Vert _{A+B}^{2}-\left\Vert \left(A+B\right)^{-1}\left(Aa+Bb\right)\right\Vert _{A+B}^{2}+\left\Vert a\right\Vert _{A}^{2}+\left\Vert b\right\Vert _{B}^{2}\\
 & = & \left\Vert x-\left(A+B\right)^{-1}\left(Aa+Bb\right)\right\Vert _{A+B}^{2}+\underbrace{\left\Vert a\right\Vert _{A}^{2}+\left\Vert b\right\Vert _{B}^{2}-\left\Vert Aa+Bb\right\Vert _{\left(A+B\right)^{-1}}^{2}}_{*}
\end{eqnarray*}
\begin{eqnarray*}
* & = & \left\Vert a\right\Vert _{A}^{2}+\left\Vert b\right\Vert _{B}^{2}-\left\Vert Aa\right\Vert _{\left(A+B\right)^{-1}}^{2}-a^{T}A\left(A+B\right)^{-1}Bb-b^{T}B\left(A+B\right)^{-1}Aa-\left\Vert Bb\right\Vert _{\left(A+B\right)^{-1}}^{2}\\
 & = & a^{T}A\left(a-\left(A+B\right)^{-1}Aa-\left(A+B\right)^{-1}Bb\right)+b^{T}B\left(b-\left(A+B\right)^{-1}Bb-\left(A+B\right)^{-1}Aa\right)\\
 & = & a^{T}A\left(A+B\right)^{-1}B\left(a-b\right)+b^{T}B\left(A+B\right)^{-1}A\left(b-a\right)\\
 & = & a^{T}\left(B^{-1}+A^{-1}\right)^{-1}\left(a-b\right)+b^{T}\left(A^{-1}+B^{-1}\right)^{-1}\left(b-a\right)\\
 & = & \left\Vert a-b\right\Vert _{\left(A^{-1}+B^{-1}\right)^{-1}}^{2}\\
 & = & \left\Vert a-b\right\Vert _{A\left(A+B\right)^{-1}B}^{2}
\end{eqnarray*}
By continuity, the claim also holds when $A,B$ are not both invertible,
provided $\left(A+B\right)$ is invertible.
\end{proof}
Computing the expectations in \eqref{eq:moments-closed-loop} involves
computation of integrals containing $P\left(x,t\right)\lambda_{t}\left(\theta,x\right)$.
Taking the ADF approximation $P\left(x,t\right)\approx\mathcal{N}\left(x;\mu_{t},\Sigma_{t}\right)$,
and using the claim above, we have 
\begin{eqnarray*}
P\left(x,t\right)\lambda_{t}\left(\theta,x\right) & \approx & \lambda^{0}\sqrt{\frac{\left|F\right|}{\left(2\pi\right)^{m}}}\mathcal{N}\left(x;\mu_{t},\Sigma_{t}\right)\exp\left(-\frac{1}{2}\left(\left\Vert Hx-c\right\Vert _{M}^{2}+\left\Vert \theta-\left(F+R\right)^{-1}\left(Fc+RHx\right)\right\Vert _{F+R}^{2}\right)\right)\\
 & = & \lambda^{0}\sqrt{\frac{\left|F\right|}{\left(2\pi\right)^{m+n}\left|\Sigma_{t}\right|}}\exp\left(-\frac{1}{2}\left\Vert x-\mu_{t}\right\Vert _{Q_{t}}^{2}\right)\\
 &  & \times\exp\left(-\frac{1}{2}\left(\left\Vert x-\bar{c}\right\Vert _{H^{T}MH}^{2}+\left\Vert \theta-\left(F+R\right)^{-1}\left(Fc+RHx\right)\right\Vert _{F+R}^{2}\right)\right)\\
 & = & \lambda^{0}\sqrt{\frac{\left|F\right|}{\left(2\pi\right)^{m+n}\left|\Sigma_{t}\right|}}\exp\left(-\frac{1}{2}\left(\left\Vert \mu_{t}-\bar{c}\right\Vert _{Q_{t}^{M}}^{2}+\left\Vert x-\mu_{t}^{M}\right\Vert _{Q_{t}+H^{T}MH}^{2}\right)\right)\\
 &  & \times\exp\left(-\frac{1}{2}\left\Vert \theta-\left(F+R\right)^{-1}\left(Fc+RHx\right)\right\Vert _{F+R}^{2}\right)
\end{eqnarray*}
where $H_{r}^{-1}$ is any right inverse of $H$, and 
\begin{eqnarray*}
\bar{c} & \triangleq & H_{r}^{-1}c\\
M & \triangleq & F\left(F+R\right)^{-1}R=\left(F^{-1}+R^{-1}\right)^{-1}\\
\mu_{t}^{M} & \triangleq & \left(Q_{t}+H^{T}MH\right)^{-1}\left(Q_{t}\mu_{t}+H^{T}MH\bar{c}\right)\\
Q_{t}^{M} & \triangleq & Q_{t}\left(Q_{t}+H^{T}MH\right)^{-1}H^{T}MH=\left(I+H^{T}MH\Sigma_{t}\right)^{-1}H^{T}MH.
\end{eqnarray*}
An alternate form for $Q_{t}^{M}$ may be derived from the Woodbury
identity as follows, 
\begin{eqnarray}
Q_{t}^{M} & = & \left(I+H^{T}MH\Sigma_{t}\right)^{-1}H^{T}MH\nonumber \\
 & = & \left(I-H^{T}\left(M^{-1}+H\Sigma_{t}H^{T}\right)^{-1}H\Sigma_{t}\right)H^{T}MH\nonumber \\
 & = & H^{T}\left(I-\left(M^{-1}+H\Sigma_{t}H^{T}\right)^{-1}H\Sigma_{t}H^{T}\right)MH\nonumber \\
 & = & H^{T}\left(\left(M^{-1}+H\Sigma_{t}H^{T}\right)^{-1}M^{-1}\right)MH\nonumber \\
 & = & H^{T}S_{t}^{M}H,\label{eq:QM-woodbury}
\end{eqnarray}
where 
\[
S_{t}^{M}\triangleq\left(M^{-1}+H\Sigma_{t}H^{T}\right)^{-1},
\]
so we can write
\begin{eqnarray*}
P\left(x,t\right)\lambda_{t}\left(\theta,x\right) & \approx & \lambda^{0}\sqrt{\frac{\left|F\right|}{\left(2\pi\right)^{m+n}\left|\Sigma_{t}\right|}}\exp\left(-\frac{1}{2}\left(\left\Vert H\mu_{t}-c\right\Vert _{S_{t}^{M}}^{2}+\left\Vert x-\mu_{t}^{M}\right\Vert _{Q_{t}+H^{T}MH}^{2}\right)\right)\\
 &  & \times\exp\left(-\frac{1}{2}\left\Vert \theta-\left(F+R\right)^{-1}\left(Fc+RHx\right)\right\Vert _{F+R}^{2}\right).
\end{eqnarray*}

Now we define
\[
g_{t}\triangleq\int\hat{\lambda}_{t}\left(\theta\right)d\theta=\int\mathrm{E}_{P}\left[\lambda_{t}\left(\theta,X_{t}\right)\right]d\theta,
\]
and compute its value as follows:
\begin{eqnarray*}
g_{t} & = & \int\hat{\lambda}_{t}\left(\theta\right)d\theta\\
 & = & \int\int P\left(x,t\right)\lambda_{t}\left(\theta,x\right)dxd\theta\\
 & \approx & \lambda^{0}\sqrt{\frac{\det F}{\left(2\pi\right)^{m+n}\det\Sigma_{t}}}\exp\left(-\frac{1}{2}\left\Vert H\mu_{t}-c\right\Vert _{S_{t}^{M}}^{2}\right)\int dx\exp\left(-\frac{1}{2}\left\Vert x-\mu_{t}^{M}\right\Vert _{Q_{t}+H^{T}MH}^{2}\right)\\
 &  & \times\int d\theta\exp\left(-\frac{1}{2}\left\Vert \theta-\left(F+R\right)^{-1}\left(Fc+RHx\right)\right\Vert _{F+R}^{2}\right)\\
 & = & \lambda^{0}\sqrt{\frac{\det F}{\left(2\pi\right)^{n}\det\Sigma_{t}\det\left(F+R\right)}}\exp\left(-\frac{1}{2}\left\Vert H\mu_{t}-c\right\Vert _{S_{t}^{M}}^{2}\right)\\
 &  & \int dx\exp\left(-\frac{1}{2}\left\Vert x-\mu_{t}^{M}\right\Vert _{Q_{t}+H^{T}MH}^{2}\right)\\
 & = & \lambda^{0}\sqrt{\frac{\det F}{\det\Sigma_{t}\det\left(Q_{t}+H^{T}MH\right)\det\left(F+R\right)}}\exp\left(-\frac{1}{2}\left\Vert H\mu_{t}-c\right\Vert _{S_{t}^{M}}^{2}\right)
\end{eqnarray*}
To simplify the expression under the square root, we note that
\[
\frac{\det F}{\det\left(F+R\right)}=\det\left(F\left(F+R\right)^{-1}\right)=\det\left(\left(R^{-1}+F^{-1}\right)^{-1}R^{-1}\right)=\frac{\det M}{\det R},
\]
and, using the matrix determinant lemma and \eqref{eq:QM-woodbury}
\[
\det\left(Q_{t}+H^{T}MH\right)=\det Q_{t}\det M\det\left(M^{-1}+H\Sigma_{t}H^{T}\right)=\frac{\det M}{\det\Sigma_{t}\det S_{t}^{M}}
\]
so 
\[
g_{t}=\lambda^{0}\sqrt{\frac{\det S_{t}^{M}}{\det R}}\exp\left(-\frac{1}{2}\left\Vert H\mu_{t}-c\right\Vert _{S_{t}^{M}}^{2}\right).
\]

A similar computation yields additional terms from \eqref{eq:moments-closed-loop},
expressed in terms of $g_{t}$.
\begin{eqnarray*}
\mathrm{E}_{P}\left[X_{t}\int\lambda_{t}\left(\theta,X_{t}\right)d\theta\right] & = & \int dx\int d\theta xP\left(x,t\right)\lambda_{t}\left(\theta,x\right)\\
 & \approx & \lambda^{0}\sqrt{\frac{\det F}{\left(2\pi\right)^{m+n}\det\Sigma_{t}}}\exp\left(-\frac{1}{2}\left\Vert H\mu_{t}-c\right\Vert _{S_{t}^{M}}^{2}\right)\\
 &  & \times\int dx\,x\exp\left(-\frac{1}{2}\left\Vert x-\mu_{t}^{M}\right\Vert _{Q_{t}+H^{T}MH}^{2}\right)\\
 &  & \times\int d\theta\exp\left(-\frac{1}{2}\left\Vert \theta-\left(F+R\right)^{-1}\left(Fc+RHx\right)\right\Vert _{F+R}^{2}\right)\\
 & = & \lambda^{0}\sqrt{\frac{\det F}{\left(2\pi\right)^{n}\det\Sigma_{t}\det\left(F+R\right)}}\exp\left(-\frac{1}{2}\left\Vert H\mu_{t}-c\right\Vert _{S_{t}^{M}}^{2}\right)\\
 &  & \times\int dx\,x\exp\left(-\frac{1}{2}\left\Vert x-\mu_{t}^{M}\right\Vert _{Q_{t}+H^{T}MH}^{2}\right)\\
 & = & \lambda^{0}\sqrt{\frac{\det F}{\det\Sigma_{t}\det\left(Q_{t}+H^{T}MH\right)\det\left(F+R\right)}}\exp\left(-\frac{1}{2}\left\Vert H\mu_{t}-c\right\Vert _{S_{t}^{M}}^{2}\right)\mu_{t}^{M}\\
 & = & g_{t}\mu_{t}^{M},
\end{eqnarray*}

\begin{eqnarray*}
\mathrm{E}_{P}\left[\tilde{X}_{t}\tilde{X}_{t}^{T}\int\lambda_{t}\left(\theta,X_{t}\right)d\theta\right] & = & \int dx\int d\theta P\left(x,t\right)\left(x-\mu_{t}\right)\left(x-\mu_{t}\right)^{T}\lambda_{t}\left(\theta,x\right)d\theta\\
 & \approx & \lambda^{0}\sqrt{\frac{\det F}{\left(2\pi\right)^{m+n}\det\Sigma_{t}}}\exp\left(-\frac{1}{2}\left\Vert H\mu_{t}-c\right\Vert _{S_{t}^{M}}^{2}\right)\\
 &  & \times\int dx\,\left(x-\mu_{t}\right)\left(x-\mu_{t}\right)^{T}\exp\left(-\frac{1}{2}\left\Vert x-\mu_{t}^{M}\right\Vert _{Q_{t}+H^{T}MH}^{2}\right)\\
 &  & \times\int d\theta\exp\left(-\frac{1}{2}\left\Vert \theta-\left(F+R\right)^{-1}\left(Fc+RHx\right)\right\Vert _{F+R}^{2}\right)\\
 & = & \lambda^{0}\sqrt{\frac{\det F}{\left(2\pi\right)^{n}\det\Sigma_{t}\det\left(F+R\right)}}\exp\left(-\frac{1}{2}\left\Vert H\mu_{t}-c\right\Vert _{S_{t}^{M}}^{2}\right)\\
 &  & \times\int dx\,\left(x-\mu_{t}\right)\left(x-\mu_{t}\right)^{T}\exp\left(-\frac{1}{2}\left\Vert x-\mu_{t}^{M}\right\Vert _{Q_{t}+H^{T}MH}^{2}\right)\\
 & = & \lambda^{0}\sqrt{\frac{\det F}{\det\Sigma_{t}\det\left(Q_{t}+H^{T}MH\right)\det\left(F+R\right)}}\\
 &  & \times\exp\left(-\frac{1}{2}\left\Vert H\mu_{t}-c\right\Vert _{S_{t}^{M}}^{2}\right)\left[\left(Q_{t}+H^{T}MH\right)^{-1}+\left(\mu_{t}-\mu_{t}^{M}\right)\left(\mu_{t}-\mu_{t}^{M}\right)^{T}\right]\\
 & = & g_{t}\left[\left(Q_{t}+H^{T}MH\right)^{-1}+\left(\mu_{t}-\mu_{t}^{M}\right)\left(\mu_{t}-\mu_{t}^{M}\right)^{T}\right].
\end{eqnarray*}
Assuming $X$ has linear dynamics, substituting these results into
\eqref{eq:moments-closed-loop} yields the following filtering equations
between spikes (we abuse notation and use $\mu_{t},\Sigma_{t}$ to
refer to the ADF-approximate quantities from here on), 
\begin{eqnarray}
\frac{d\mu_{t}}{dt} & = & A\mu_{t}+B\left(U_{t}\right)+g_{t}\left(\mu_{t}-\mu_{t}^{M}\right),\nonumber \\
\frac{d\Sigma_{t}}{dt} & = & A\Sigma_{t}+\Sigma_{t}A+DD^{T}\nonumber \\
 &  & +g_{t}\left[\Sigma_{t}-\left(Q_{t}+H^{T}MH\right)^{-1}-\left(\mu_{t}-\mu_{t}^{M}\right)\left(\mu_{t}-\mu_{t}^{M}\right)^{T}\right].\label{eq:adf-intermediate}
\end{eqnarray}
We simplify this by computing
\begin{eqnarray*}
\mu_{t}-\mu_{t}^{M} & = & \mu_{t}-\left(Q_{t}+H^{T}MH\right)^{-1}\left(Q_{t}\mu_{t}+H^{T}MH\bar{c}\right)\\
 & = & \left(Q_{t}+H^{T}MH\right)^{-1}H^{T}MH\left(\mu_{t}-\bar{c}\right)\\
 & = & \Sigma_{t}Q_{t}^{M}\left(\mu_{t}-\bar{c}\right)\\
 & = & \Sigma_{t}H^{T}S_{t}^{M}\left(H\mu_{t}-c\right),\\
\Sigma_{t}-\left(Q_{t}+H^{T}MH\right)^{-1} & = & \Sigma_{t}\left(I-\left(I+H^{T}MH\Sigma_{t}\right)^{-1}\right)\\
 & = & \Sigma_{t}H^{T}\left(M^{-1}+H\Sigma_{t}H^{T}\right)^{-1}H\Sigma_{t}\\
 & = & \Sigma_{t}Q_{t}^{M}\Sigma_{t}\,,
\end{eqnarray*}
where we have used the Woodbury identity and \eqref{eq:QM-woodbury}.
Substituting into \eqref{eq:adf-intermediate} we obtain the form

\begin{eqnarray}
\frac{d\mu_{t}}{dt} & = & A\mu_{t}+B\left(U_{t}\right)+g_{t}\Sigma_{t}H^{T}S_{t}^{M}\left(H\mu_{t}-c\right)\nonumber \\
\frac{d\Sigma_{t}}{dt} & = & A\Sigma_{t}+\Sigma_{t}A+DD^{T}\nonumber \\
 &  & +g_{t}\left[\Sigma_{t}H^{T}S_{t}^{M}H\Sigma_{t}-\Sigma_{t}H^{T}S_{t}^{M}\left(H\mu_{t}-c\right)\left(H\mu_{t}-c\right)^{T}S_{t}^{M}H\Sigma_{t}\right].\label{eq:adf-inter-spike}
\end{eqnarray}

\subsubsection{Effect of spikes}

When a spike occurs at time $t$ in preferred location $\theta$,
the update according to \eqref{eq:rhodes-snyder-time-variant} is
\begin{eqnarray*}
P\left(x,t^{+}\right) & = & P\left(x,t^{-}\right)+P\left(x,t^{-}\right)\frac{\lambda_{t^{-}}\left(\theta,x\right)-\hat{\lambda}_{t^{-}}\left(\theta\right)}{\hat{\lambda}_{t^{-}}\left(\theta\right)}\\
 & = & P\left(x,t^{-}\right)\frac{\lambda_{t^{-}}\left(\theta,x\right)}{\hat{\lambda}_{t^{-}}\left(\theta\right)}\\
 & = & \frac{P\left(x,t^{-}\right)\lambda_{t^{-}}\left(\theta,x\right)}{\int P\left(x,t^{-}\right)\lambda_{t^{-}}\left(\theta,x\right)dx}\,.
\end{eqnarray*}
To compute this sum we note that $P\left(x,t\right)\lambda_{t}\left(\theta,x\right)$,
under the ADF approximation, may be written as a single Gaussian in
$x$,
\begin{eqnarray*}
P\left(x,t\right)\lambda_{t}\left(\theta,x\right)dx & \approx & \lambda^{0}\sqrt{\frac{\det F}{\left(2\pi\right)^{m+n}\det\Sigma_{t}}}\exp\left(-\frac{1}{2}\left\Vert x-\mu_{t}\right\Vert _{Q_{t}}^{2}-\frac{1}{2}\left\Vert \theta-c\right\Vert _{F}^{2}-\frac{1}{2}\left\Vert Hx-\theta\right\Vert _{R}^{2}\right)\\
 & = & \lambda^{0}\sqrt{\frac{\det F}{\left(2\pi\right)^{m+n}\det\Sigma_{t}}}\exp\left(-\frac{1}{2}\left\Vert \theta-c\right\Vert _{F}^{2}-\frac{1}{2}\left\Vert x-\mu_{t}\right\Vert _{Q_{t}}^{2}-\frac{1}{2}\left\Vert x-H_{r}^{-1}\theta\right\Vert _{H^{T}RH}^{2}\right)\\
 & = & C_{t}\left(\theta\right)\cdot\exp\left(-\frac{1}{2}\left\Vert x-\left(Q_{t}+H^{T}RH\right)^{-1}\left(Q_{t}\mu_{t}+H^{T}R\theta\right)\right\Vert _{Q_{t}+H^{T}RH}^{2}\right),
\end{eqnarray*}
where 
\begin{eqnarray*}
C_{t}\left(\theta\right) & = & \lambda^{0}\sqrt{\frac{\det F}{\left(2\pi\right)^{m+n}\det\Sigma_{t}}}\exp\left(-\frac{1}{2}\left\Vert \theta-c\right\Vert _{F}^{2}-\frac{1}{2}\left\Vert H_{r}^{-1}\theta-\mu_{t}\right\Vert _{Q_{t}^{R}}^{2}\right)\\
Q_{t}^{R} & \triangleq & Q_{t}\left(Q_{t}+H^{T}RH\right)^{-1}H^{T}RH=\left(I+H^{T}RH\Sigma_{t}\right)^{-1}H^{T}RH.
\end{eqnarray*}
Analogously to the computation for $Q_{t}^{M}$ above, we have $Q_{t}^{R}=H^{T}S_{t}^{R}H,$
where
\[
S_{t}^{R}\triangleq\left(R^{-1}+H\Sigma_{t}H^{T}\right)^{-1},
\]

Now $P\left(x,t^{+}\right)$ is given by the normalized Gaussian,
\begin{eqnarray*}
P\left(x,t^{+}\right) & = & \frac{P\left(x,t^{-}\right)\lambda_{t^{-}}\left(\theta,x\right)}{\int P\left(x,t^{-}\right)\lambda_{t^{-}}\left(\theta,x\right)dx}\\
 & = & \sqrt{\frac{\det\left(\Sigma_{t^{-}}^{-1}+H^{T}RH\right)}{\left(2\pi\right)^{n}}}\exp\left(-\frac{1}{2}\left\Vert x-\left(\Sigma_{t^{-}}+H^{T}RH\right)^{-1}\left(\Sigma_{t^{-}}^{-1}\mu_{t^{-}}+H^{T}R\theta\right)\right\Vert _{\Sigma_{t^{-}}^{-1}+H^{T}RH}^{2}\right)\\
 & = & \mathcal{N}\left(x,\left(\Sigma_{t^{-}}^{-1}+H^{T}RH\right)^{-1}\left(\Sigma_{t^{-}}^{-1}\mu_{t^{-}}+H^{T}R\theta\right),\left(\Sigma_{t^{-}}^{-1}+H^{T}RH\right)^{-1}\right),
\end{eqnarray*}
and the update is
\begin{eqnarray*}
\mu_{t^{+}} & = & \left(\Sigma_{t^{-}}^{-1}+H^{T}RH\right)^{-1}\left(\Sigma_{t^{-}}^{-1}\mu_{t^{-}}+H^{T}R\theta\right)\\
\Sigma_{t^{+}} & = & \left(\Sigma_{t^{-}}^{-1}+H^{T}RH\right)^{-1}.
\end{eqnarray*}
To incorporate these updates into the inter-spike SDE \eqref{eq:adf-inter-spike}
they can be cast in the form
\begin{eqnarray*}
\mu_{t^{+}} & = & \mu_{t^{-}}+\left(\Sigma_{t^{-}}^{-1}+H^{T}RH\right)^{-1}H^{T}RH\left(H_{r}^{-1}\theta-\mu_{t^{-}}\right)\\
 & = & \mu_{t^{-}}+\Sigma_{t^{-}}Q_{t^{-}}^{R}\left(H_{r}^{-1}\theta-\mu_{t^{-}}\right)\\
 & = & \mu_{t^{-}}+\Sigma_{t^{-}}H^{T}S_{t^{-}}^{R}\left(\theta-H\mu_{t^{-}}\right),\\
\Sigma_{t^{+}} & = & \Sigma_{t^{-}}-\left(\Sigma_{t^{-}}^{-1}+H^{T}RH\right)^{-1}H^{T}RH\Sigma_{t^{-}}\\
 & = & \Sigma_{t^{-}}-\Sigma_{t^{-}}Q_{t^{-}}^{R}\Sigma_{t^{-}}\\
 & = & \Sigma_{t^{-}}-\Sigma_{t^{-}}H^{T}S_{t^{-}}^{R}H\Sigma_{t^{-}}\,,
\end{eqnarray*}
giving the full filtering SDE
\begin{eqnarray}
d\mu_{t} & = & A\mu_{t}dt+B\left(U_{t}\right)dt+g_{t}\Sigma_{t}H^{T}S_{t}^{M}\left(H\mu_{t}-c\right)dt+\Sigma_{t^{-}}H^{T}S_{t^{-}}^{R}\int_{\theta\in\mathbb{R}^{m}}\left(\theta-H\mu_{t^{-}}\right)N\left(dt\times d\theta\right),\nonumber \\
d\Sigma_{t} & = & \left(A\Sigma_{t}+\Sigma_{t}A+DD^{T}+g_{t}\left[\Sigma_{t}H^{T}S_{t}^{M}H\Sigma_{t}-\Sigma_{t}H^{T}S_{t}^{M}\left(H\mu_{t}-c\right)\left(H\mu_{t}-c\right)^{T}S_{t}^{M}H\Sigma_{t}\right]\right)dt\nonumber \\
 &  & -\Sigma_{t^{-}}H^{T}S_{t^{-}}^{R}H\Sigma_{t^{-}}dN_{t}.\label{eq:adf-full}
\end{eqnarray}

\subsection{Non-linear dynamics}

In case of non-linear dynamics
\[
dX_{t}=\left(A\left(X_{t}\right)+B\left(U_{t}\right)\right)dt+D_{t}dW_{t}
\]
the ADF approximation may also be applied to the terms involving $A\left(X_{t}\right)$
in \eqref{eq:moments-closed-loop}. Assume $A^{\left(i\right)}$,
the $i$-th element of $A$, is given by a power series around $\mu_{t}$,
written in multi-index notation,
\[
A^{\left(i\right)}\left(x\right)=\sum_{\alpha}A_{\alpha}^{\left(i\right)}\left(\mu_{t}\right)\left(x-\mu_{t}\right)^{\alpha},
\]
where the sum is over all multi-indices $\alpha$. Then, assuming
the ADF approximation $X_{t}\sim\mathcal{N}\left(\mu_{t},\Sigma_{t}\right)$,
\[
\mathrm{E}_{P}\left[A^{\left(i\right)}\left(X_{t}\right)\right]=\sum_{\alpha}A_{\alpha}^{\left(i\right)}\left(\mu_{t}\right)\mathrm{E}_{\alpha}\left(\Sigma_{t}\right),
\]
where $\mathrm{E}_{\alpha}\left(\Sigma\right)$ is defined as $\mathrm{E}\left(Z^{\alpha}\right)=\mathrm{E}\prod_{k}Z_{k}^{\alpha_{k}}$
for $Z\sim\mathcal{N}\left(0,\Sigma\right)$, and may be computed
from Isserlis' theorem. Similarly,
\begin{eqnarray*}
\mathrm{E}_{P}\left[A\left(X_{t}\right)\tilde{X}_{t}^{T}\right]_{ij} & = & \mathrm{E}_{P}\left[A^{\left(i\right)}\left(X_{t}\right)\left(X_{t}^{\left(j\right)}-\mu_{t}^{\left(j\right)}\right)\right]\\
 & = & \sum_{\alpha}A_{\alpha}^{\left(i\right)}\left(\mu_{t}\right)\mathrm{E}_{\alpha+e_{j}}\left(\Sigma_{t}\right),
\end{eqnarray*}
where $e_{j}$ is $j$-th standard basis vector (the multi-index corresponding
to the single index $j$).

Writing $A_{\alpha}=\left(A_{\alpha}^{\left(1\right)},\ldots A_{\alpha}^{\left(n\right)}\right)^{T}$
and $\mathbf{E}_{\alpha,t}=\left(\mathrm{E}_{\alpha+e_{1}}\left(\Sigma_{t}\right),\ldots,\mathrm{E}_{\alpha+e_{n}}\left(\Sigma_{t}\right)\right)$
the filtering equations become
\begin{eqnarray*}
d\mu_{t} & = & \sum_{\alpha}A_{\alpha}\left(\mu_{t}\right)\mathrm{E}_{\alpha}\left(\Sigma_{t}\right)+B\left(U_{t}\right)dt+\\
 &  & g_{t}\Sigma_{t}Q_{t}^{RF}\left(\mu_{t}-\bar{c}\right)dt+\Sigma_{t^{-}}Q_{t^{-}}^{R}\int_{\theta\in\mathbb{R}^{n}}\left(H_{r}^{-1}\theta-\mu_{t^{-}}\right)N\left(dt\times d\theta\right)\\
d\Sigma_{t} & = & \left(\sum_{\alpha}\left(A_{\alpha}\left(\mu_{t}\right)\mathbf{E}_{\alpha,t}^{T}+\mathbf{E}_{\alpha,t}A_{\alpha}\left(\mu_{t}\right)^{T}\right)+DD^{T}+g_{t}\left[\Sigma_{t}Q_{t}^{RF}\Sigma_{t}-\Sigma_{t}Q_{t}^{RF}\left(\mu_{t}-\bar{c}\right)\left(\mu_{t}-\bar{c}\right)^{T}Q_{t}^{RF}\Sigma_{t}\right]\right)dt\\
 &  & -\Sigma_{t^{-}}Q_{t^{-}}^{R}\Sigma_{t^{-}}dN_{t}.
\end{eqnarray*}
Analogous comments apply when the noise gain $D_{t}$ is a non-linear
function $D\left(X_{t}\right)$, provided each element $\left[D\left(x\right)D\left(x\right)^{T}\right]_{ij}$
may be expanded as a power series.

\subsection{Comparison of estimated posterior variance and MSE}

In the main text, we studied optimal encoding using the posterior
variance as a proxy for the MSE. Letting $\mu_{t},\Sigma_{t}$ denote
the approximate posterior moments given by the filter, the MSE and
posterior variance are related as follows,
\begin{eqnarray*}
\mathrm{MSE}_{t} & \triangleq & \mathrm{E}\left[\mathrm{tr}\left(X_{t}-\mu_{t}\right)\left(X_{t}-\mu_{t}\right)^{T}\right]=\mathrm{E}\mathrm{E}_{P}^{t}\mathrm{tr}\left(X_{t}-\mu_{t}\right)\left(X_{t}-\mu_{t}\right)^{T}\\
 & = & \mathrm{E}\left[\mathrm{tr}\left(\mathrm{Var}_{P}^{t}X_{t}\right)\right]+\mathrm{E}\left[\mathrm{tr}\left(\mu_{t}-\mathrm{E}_{P}^{t}X_{t}\right)\left(\mu_{t}-\mathrm{E}_{P}^{t}X_{t}\right)^{T}\right],
\end{eqnarray*}
where $\mathrm{E}_{P}^{t}\left[\cdot\right],\mathrm{Var}_{P}^{t}\left[\cdot\right]$
are resp. the mean and covariance conditioned on $\mathcal{N}_{t}$,
and $\mathrm{tr}$ is the trace operator. Thus for an exact filter,
having $\mu_{t}=\mathrm{E}_{P}^{t}X_{t},\Sigma_{t}=\mathrm{Cov}_{P}^{t}X_{t}$,
we would have $\mathrm{MSE}_{t}=\mathrm{trace[E}(\Sigma_{t})]$. Conversely,
if we find that $\mathrm{MSE}_{t}\approx\mathrm{trace[E}\Sigma_{t}]$,
it suggests that the errors are small (though this is not guaranteed,
since the errors in $\mu_{t}$ and $\Sigma_{t}$ may effect the MSE
in opposite directions, if the variance is underestimated).

Figure \ref{variance-mse} shows the variance and MSE in estimating
a linear one-dimensional process, after averaging across 1000 trials.
Although the posterior variance is, on average, overestimated at the
start of trials, in the steady state it approximates the square error
reasonably well.

\begin{figure}
\centering{}\includegraphics[width=0.8\textwidth]{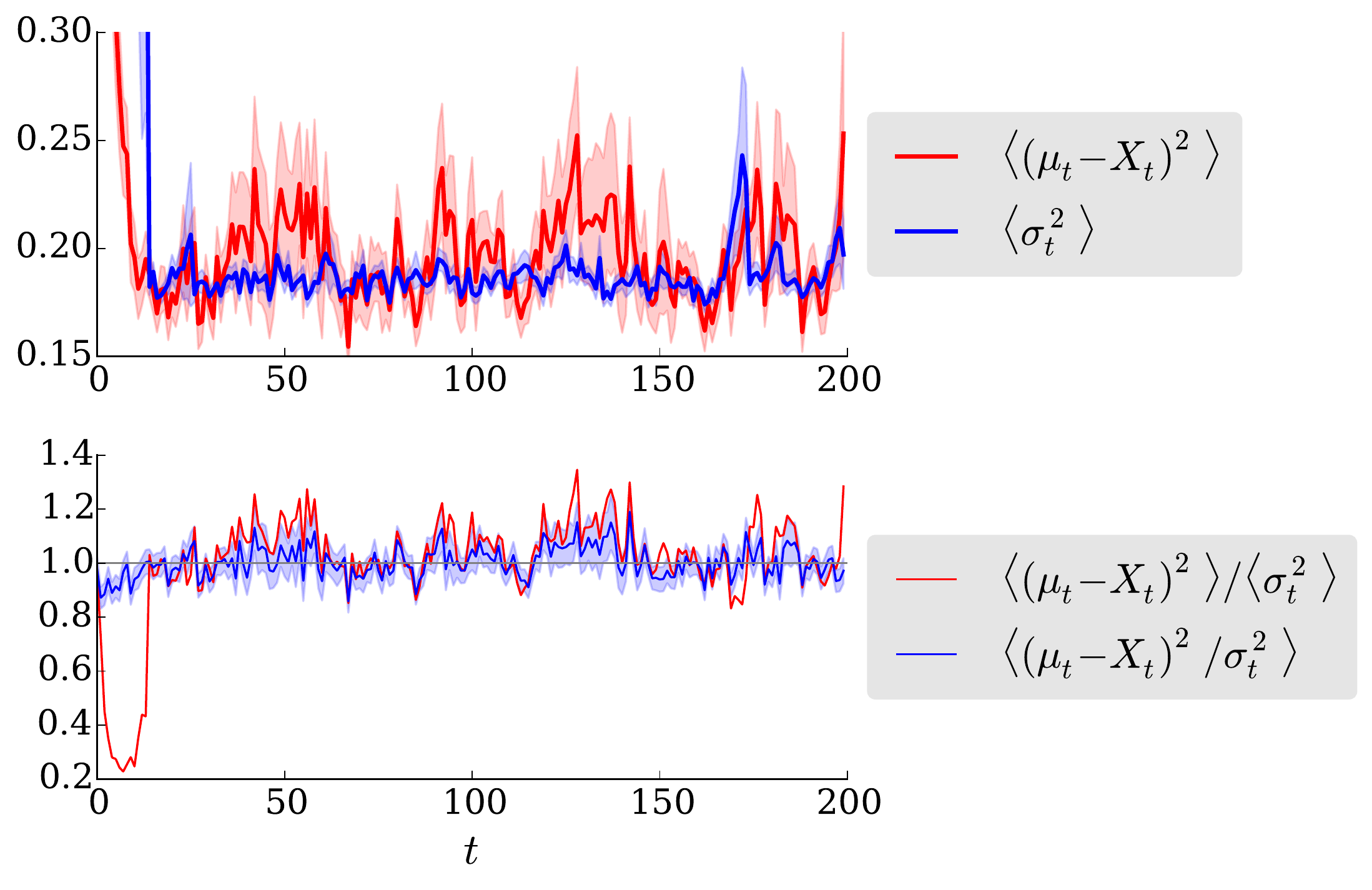}\protect\caption{Posterior variance vs. MSE when filtering a one-dimensional process
$dX_{t}=-0.1X_{t}dt+0.5dW_{t}$ (the steady-state variance of this
process is $\sigma_{0}^{2}=1.25$). The top plot shows the MSE and
mean posterior variance. The bottom plot shows the ratio of means
$\mathrm{MSE}/\left\langle \Sigma_{t}\right\rangle $ and the mean
ratio $\left\langle \mathrm{SE}/\Sigma_{t}\right\rangle $ where SE
is the squared error $\left(\mu_{t}-X_{t}\right)^{2}$. Sensory parameters
are $c=0,\sigma_{\mathrm{pop}}^{2}=0.1,\sigma_{\mathrm{tc}}^{2}=0.01,\lambda^{0}=10$.
The means were taken across 1000 trials. Shaded areas indicate error
estimates obtained as sample standard deviation divided by square
root of number of trials.}
\label{variance-mse}
\end{figure}

We also show here variants of the Figures \ref{optimal-c} and \ref{harper}
(Figures \ref{optimal-c-mse} and \ref{harper-mse}, respectively),
showing the MSE rather than the variance. The results look similar
but noisier, except in Figure \ref{mse-wide} for small population
variance, where the ADF estimation is poor due to very few spikes
occurring.

\begin{figure}
\subfloat[]{\protect\includegraphics[width=0.6\textwidth]{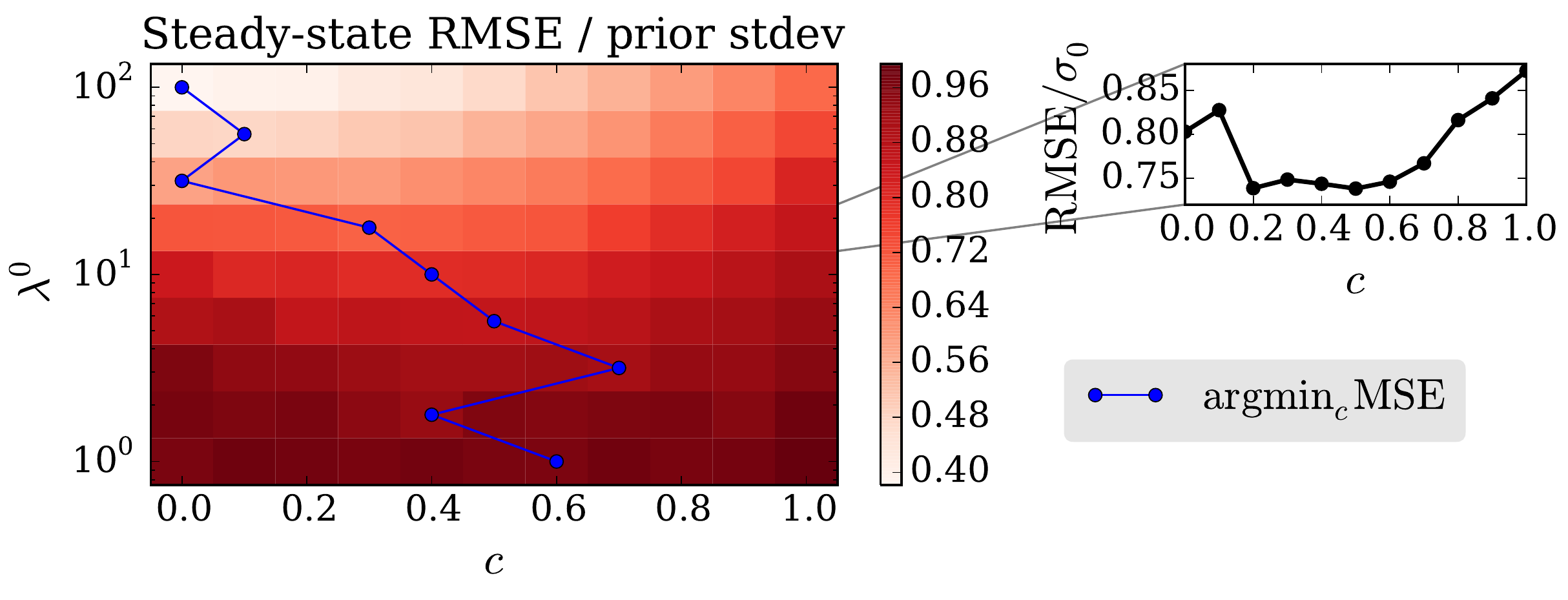}\label{mse-c-rate}

}\subfloat[]{\protect\includegraphics[width=0.35\textwidth]{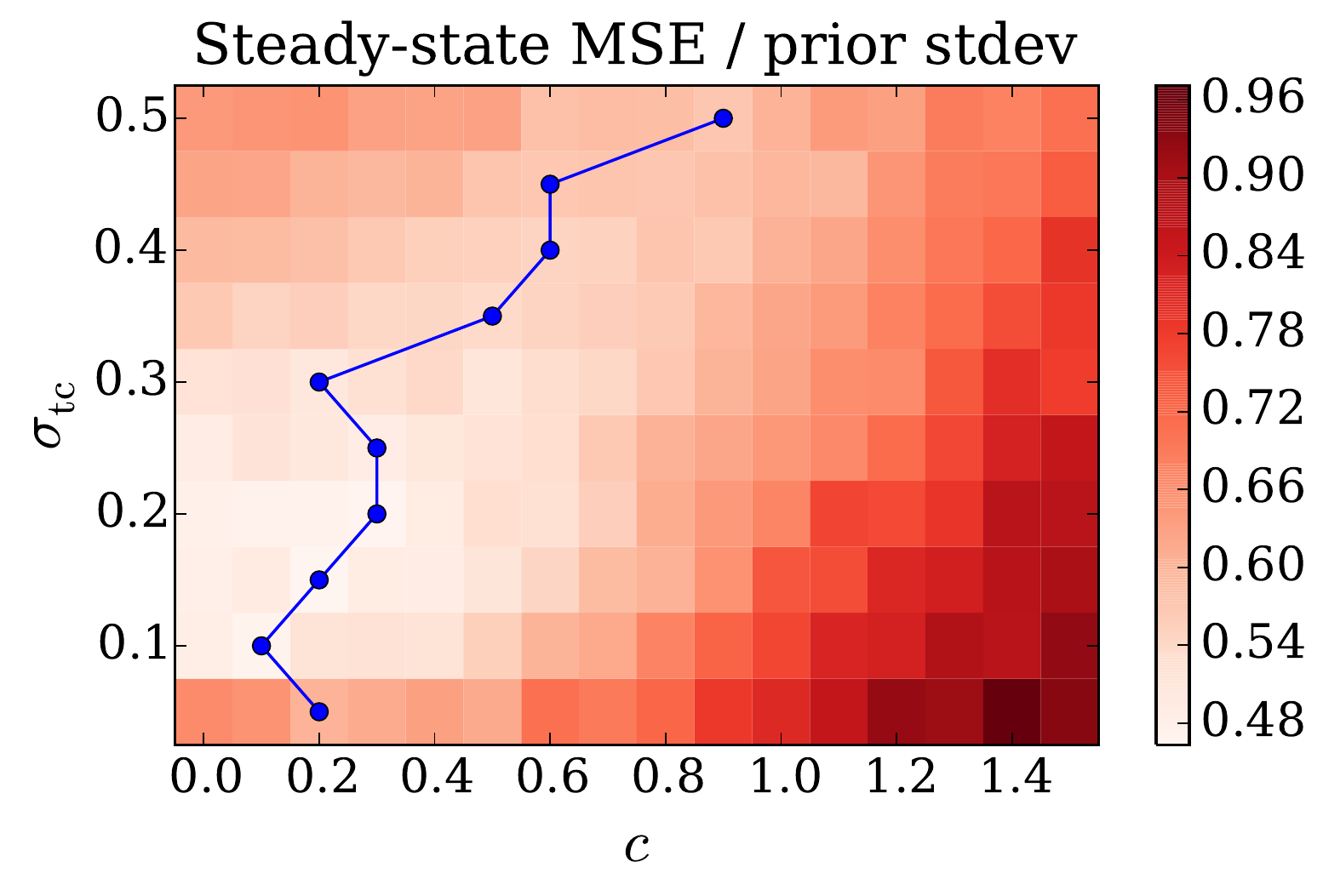}\label{mse-c-tc}
}\protect\caption{Mean Square Error as a function of model parameters. This figure is
based on the same data as Figure \ref{optimal-c}, with Root Mean
Square Error (RMSE) plotted instead of estimated posterior variance.
See Figure \ref{optimal-c} for more details.}
\label{optimal-c-mse}
\end{figure}

\begin{figure}
\subfloat[]{\protect\includegraphics[width=0.4\textwidth]{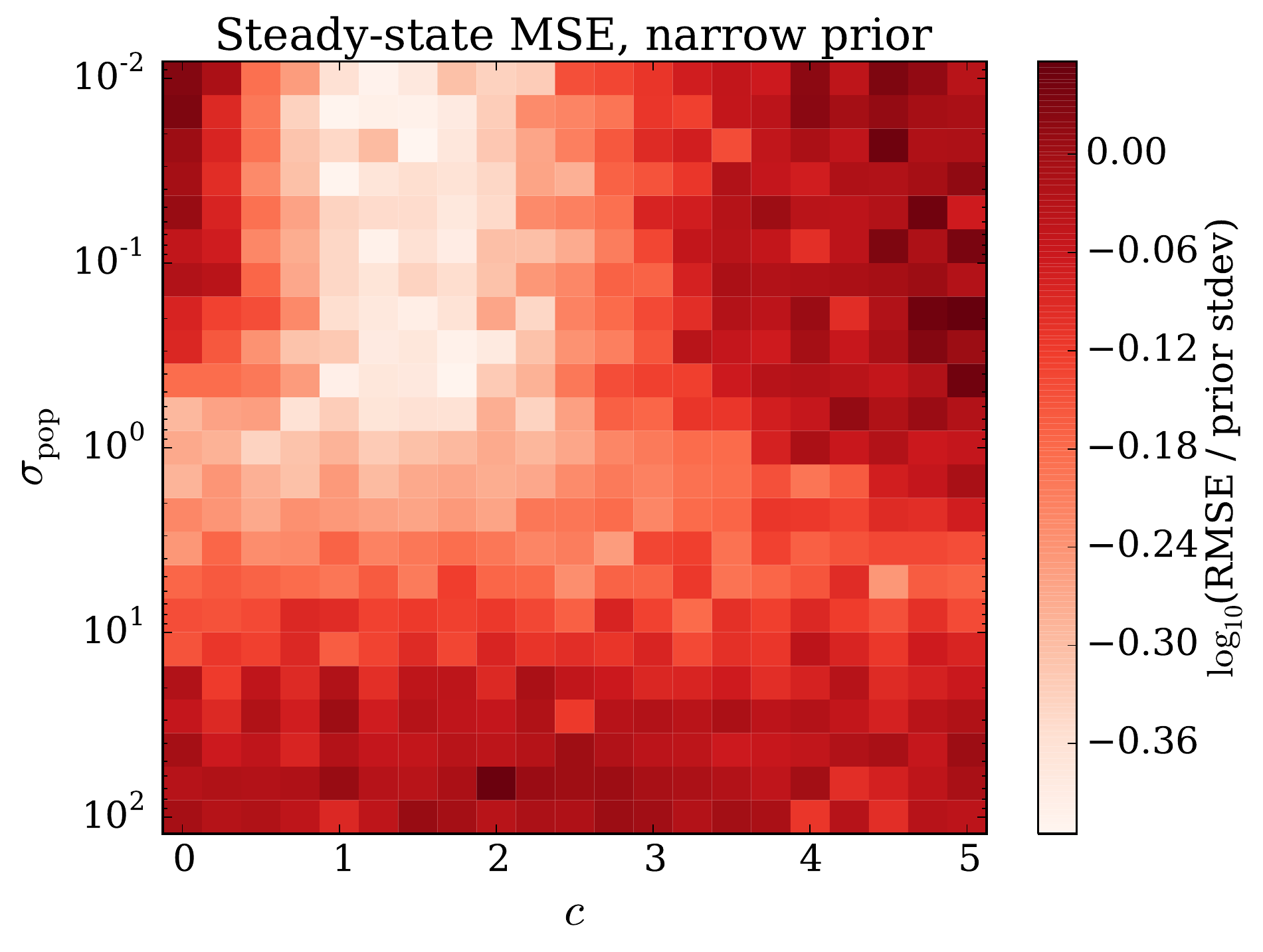}\label{mse-narrow}}\subfloat[]{\protect\includegraphics[width=0.4\textwidth]{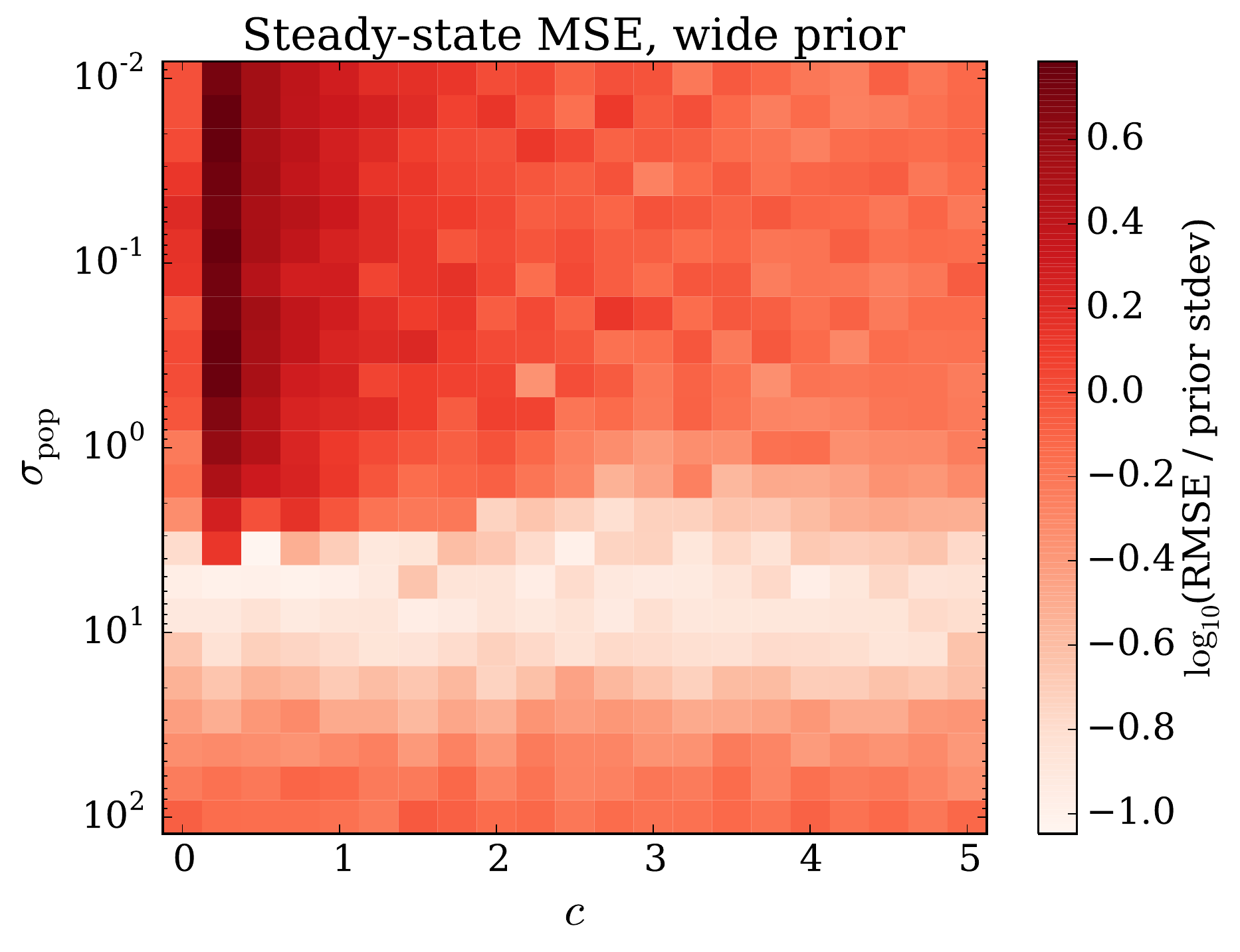}\label{mse-wide}}\protect\caption{Optimal population distribution depends on prior variance relative
to tuning curve width. This figure is based on the same data as Figure
\ref{harper}, with MSE plotted instead of estimated posterior variance.
See Figure \ref{harper} for more details.}
\label{harper-mse}
\end{figure}

{\small\bibliographystyle{unsrt}
\bibliography{ADF_NIPS}
}
\end{document}